\documentclass[12pt]{article}
\usepackage{fullpage}
\usepackage{mathrsfs}
\usepackage{float}
\usepackage{amsmath, amsfonts, amssymb, amsthm, amsbsy, amscd, bm, bbm}
\usepackage{booktabs}
    
\usepackage{url}
\usepackage{ifthen}
\usepackage[toc,page]{appendix}

\usepackage{graphicx}
\usepackage{algorithmic}
\usepackage[linesnumbered,ruled,vlined]{algorithm2e}
\usepackage{hyperref}
\usepackage{balance}
\usepackage{subfig}
\usepackage{color}
\usepackage{paralist}
\usepackage{comment}
\usepackage{cases}
\usepackage{balance}
\usepackage{kotex}

\newtheorem{theorem}{\bf Theorem}
\newtheorem{definition}{\bf Definition}
\newtheorem{lemma}{Lemma}

\newtheorem{proposition}{\bf Proposition}
\newtheorem{claim}{\bf Claim}

\renewcommand{\paragraph}[1]{\smallskip \noindent{\bf #1.} }
\newcommand{\EXP}{\mathbb{E}}
\newcommand{\PROB}{\mathbb{P}}
\DeclareMathOperator*{\argmax}{argmax}

\newcommand{\regret}{\textnormal{Regret}}

\newcommand{\set}[1]{\ensuremath{\mathcal #1}}

\newcommand{\separator}{
  \begin{center}
    \rule{\columnwidth}{0.3mm}
  \end{center}
}

\newcommand{\prob}[1]{\mathbb{P}[ #1 ]}

\newcommand{\beq}{\begin{eqnarray*}}
\newcommand{\eeq}{\end{eqnarray*}}
\newcommand{\beqn}{\begin{eqnarray}}
\newcommand{\eeqn}{\end{eqnarray}}
\newcommand{\bemn}{\begin{multiline}}
\newcommand{\eemn}{\end{multiline}}

\def\R{\mathbb{R}}

\title{Multi-armed Bandit Algorithm against Strategic Replication}
\author{
Suho Shin \thanks{LINE Plus Corporation, email: suhoyah@gmail.com.}
\and
Seungjoon Lee  \thanks{Department of Computer Science, Pohang University of Science and Technology, email: dltmdwns777@gmail.com.}
\and
Jungseul Ok \thanks{Department of Computer Science and Engineering and Graduate School of Artificial Intelligence, Pohang University of Science and Technology, email: ockjs1@gmail.com}
}

\begin{document}
\maketitle
%

%

\begin{abstract}
We consider a multi-armed bandit problem
in which a set of arms is registered by each agent, and the agent receives 
reward when its arm is selected.
An agent might strategically submit more arms with replications, 
which can bring more reward
by abusing the bandit algorithm's exploration-exploitation balance.
Our analysis reveals that
a standard algorithm indeed fails at preventing replication
and suffers from linear regret in time $T$.
We aim to design a bandit algorithm which demotivates replications
and also achieves a small cumulative regret.
We devise Hierarchical UCB (H-UCB) of replication-proof, which 
has $O(\ln T)$-regret under any equilibrium.
We further propose Robust Hierarchical UCB (RH-UCB)
which has a sublinear regret even in a realistic scenario with irrational agents replicating careless. 
We verify our theoretical findings through numerical experiments.
\end{abstract}

\section{Introduction}
\label{sec:intro}
With increasing attention to smart recommender system, 
multi-armed bandit (MAB) problem is extensively studied under various settings
as it captures the fundamental trade-off between exploration and exploitation.
In a standard stochastic MAB setting,
at each time step,
a principal chooses which arm to pull a given set of arms. 
Once an arm is pulled, a random reward associated to the arm is revealed and 
the principal receives the reward.
The principal's goal is to maximize the expected cumulative reward
or equivalently to minimize the cumulative regret 
which is defined by the expected loss compared to oracle policy that always pulls the optimal arm over the given time horizon.
To this end, at each decision, the principle needs to
address the trade-off between
exploring new arms with high potential
and exploiting empirically best arms with low risk.
A extensive line of works have been studied 
the fundamental limit in \cite{mab:lowerbound}
and proposed optimal algorithms achieving the limit~\cite{mab:ucb1, mab:klucb1, mab:ts1, mab:ts2}.

These algorithms work efficiently in ideal case, however, they might suffer from a strategic manipulation of the agents who register the arms in the platform.
For example if the principal runs $\varepsilon$-greedy based algorithm~\cite{mab:ucb1}, the probability of being selected at exploration phase
will certainly increase as the agent registers more arms.
This implies that the agent should replicate their contents as much as possible to increase their revenue. As a consequence, it will interrupt the platform's algorithm in identifying the best contents and possibly result in a decrease of the platform's revenue.



To prevent such abuse of the players, on one hand, the platform can introduce an automated process to detect the duplicated contents and blacklist the owners~\cite{YoutubeReplication, YoutubeMonetization}.
However, these process are often discouraged in some real-world applications since they might be erroneous or too costly to adopt since the platform is required to analyze features of registered contents.

On the other hand, one can borrow some off-the-shelf algorithms proposed in infinitely many-armed bandit literature~\cite{manyarmed:berry, manyarmed:wang} assuming a worst-case scenario. Yet, there are some limitations on it either: (i) these works usually assume certain characterizations on the pool of arms which might be discouraged in practice, and (ii) the replication of the arms itself could be harmful to the platform since it introduces an additional burden in hardware resources.


In this context, our major focus is how an algorithm can demotivate the agents in advance from strategically replicating arms, when it can only distinguish to whom each arm belongs.

Our main contributions are summarized as follows:
\begin{compactenum}[(a)]
\item To the best of our knowledge, 
this is the first study to address the issue of strategic replication in multi-armed bandit problem and mathematically analyze it.
\item We show that UCB1 is replication-prone and hence admits infinitely many 
replications for the strategic agents, which obviously consequences linear regret.
\item We propose H-UCB which mainly restrict the amount of exploration to be $O(\ln T)$ for each agent by separating agent and arm selection phase, and prove that it is replication-proof while achieving $O(\ln T)$-regret under the equilibrium.
\item To cope with practical challenges, we present RH-UCB which integrates arm sampling and enlarges the amount of agent exploration into $O(\sqrt{T}\ln T)$, and prove that it has $O(\sqrt{T}\ln T)$ regret in the presence of replicators who infinitely replicate their arms.
\item We provide numerical experiments to support our claim, and present some useful insights from it.
\end{compactenum}

In Section~\ref{sec:model}, we present our model and objectives.
We then formally define replication-proneness with a corresponding negative result of UCB1, and introduce replication-proofness in Section~\ref{sec:preliminary}.
We provide our positive results in \ref{sec:result} and present numerical experiments in Section~\ref{sec:simulation}.
All the proofs are presented in Appendix~\ref{sec:proofs} due to space limit.


\subsection{Related Work}
\label{sec:related}
We first present a line of works considering strategic behavior in various multi-armed bandit problem.
\cite{mabstrategy:braverman} study a multi-armed bandit problem when each arm itself is a strategic agent, and once the arm is selected by a principal, a random reward is realized, and the agent reports the reward to the principal, and finally this amount of reported reward is delivered to the principal. Here, the reported value can either be genuine or not with respect to the agent's strategic decision. The agents want to maximize their own cumulative expected utility defined by the summation of undelivered rewards, and the principal aims to extract maximal cumulative reward from the agents. The authors show that existing algorithms maintain a bad approximate equilibrium, and proposed a mechanism that only possesses an efficient approximate equilibrium in which the principal enjoys sub-linear regret. \cite{mabmanipulate:feng} study a similar problem but when the agents try to maximize the total expected number of selections by manipulating their random reward within a fixed amount of total manipulation cost, and show that popular algorithms achieve an intrinsic robustness.

Some of the works consider a strategic behavior of users in user-generated content platforms to motivate the strategic users to promptly contribute their content rather than delaying it~\cite{mabugc:jain}, or to contribute a high quality content which possibly requires their costly effort~\cite{mabugc:ghosh, mabugc:ghosh2, mabugc:ghosh3, mabugc:liu}. 
Initiated by~\cite{mabincentivize:kremer}, there also exists a line of works which study 
incentivizing scheme to motivate myopic selfish agent who selects the arms to explore more arms submitted in the system rather than exploiting the well-seeming ones~\cite{mabincentivize:wang,mabincentivize:bahar}, and also in  Bayesian setting~\cite{mabincentivize:mansour,mabincentivize:mansour2, mabincentivize:frazier}.
Since these works are quite distant from our focus of interest, we skip the detailed explanation.

Since we focus in designing an algorithm that is robust to the agent's repetitive replication, it seems worthwhile to present the works studying efficient algorithms in case of infinitely many arms. \cite{manyarmed:berry} study infinitely many-armed Bernoulli bandit problems when each arm's parameter is sampled from a common distribution under the assumption that the optimal arm has mean $1$. \cite{manyarmed:wang,manyarmed:carpentier} rather assume a certain characteristic in near-optimal arm's reward distribution to attain sublinear algorithms. \cite{ballooningmab:ghalme} consider the scenario where the arms are not fully given at first but become available in a sequential manner, which draws a need of new definition in regret. They show that it is necessary to have certain assumptions on the arrival time of the optimal arm to achieve sublinear regret, and provide corresponding sublinear regret algorithm. Unlike from these approaches, our work mainly focus at preventing such situation by demotivating the agents in advance.




\section{System model}
\label{sec:model}

We consider $n$ agents, where 
each agent $i \in \set{N} = [n] := \{1,2, ..., n\}$
is endowed with 
the set of unique arms, denoted by
$\set{O}_i := \{o_{i,1}, o_{i,2}, \ldots, o_{i,l(i)}\}$,
with a constant $l(i) \ge 1$.
Each arm $a \in \set{O} :=  \cup_{i\in\set{N}}\set{O}_i$
is associated with 
Bernoulli\footnote{
This can be extended to the reward distributions of single-parameter exponential family, while it can be corresponded with the plausible scenario of recommender system
in which the performance is often measured by the number of total clicks or hits per content.} reward distribution
of mean $\mu(a)$.
We consider a
one-shot game scenario with a principal and $n$ agents,
where we assume that agents can hide or replicate a part of its original arms.
More formally, 
in advance of MAB, 
each agent $i \in \set{N}$
decides and registers its strategy 
$\set{S}_i = \{s_{i,k}^{(c)}: c \in [c_{i, k}], k \in [l(i)] \}$ of support $\set{O}_i$
with parameters $(c_{i, 1}, c_{i, 2}, ..., c_{i,l(i)})$
such that for $k \in [l(i)]$,
$c_{i, k}$ is the number of copies of $o_{i, k}$ (including itself) in $\set{S}_i$
and $c_{i,k} = 0$ implies that $o_{i, k}$ is not registered at all,
where each copy $s_{i,k}^{(c)}$ has the same reward distribution 
of arm $o_{i,k}$, i.e., $\mu(o_{i,k}) = \mu(s_{i,k}^{(c)})$ for each $c \in [c_{i,k}]$.
Then, 
given ${\set{S}} := (\set{S}_1, \set{S}_2, ..., \set{S}_n)$,
at each round $t = 1,2, ..., T$, 
the principal selects one of the registered arms, denoted by $A_t$, according to an MAB algorithm $\mathfrak{A}$,
and receives reward $R_t$ drawn independently from the corresponding distribution of mean $\mu(A_t)$.
This is indeed a one-shot game as 
provided the principal's algorithm $\mathfrak{A}$,
each agent $i \in \set{N}$ simultaneously decides strategy $\set{S}_i$
and never updates it once registered.
We formally describe the behavior of strategic agents
in Section~\ref{sec:agent},
and the principal's objective in Section~\ref{sec:principal}.

\subsection{Strategic Agent}
\label{sec:agent}
For simplicity, we describe strategic agents
when agent~$i$ has the utility~$v_i$
defined by the accumulated reward from its arms in strategy $\set{S}_i$, i.e., 
$v_i(\set{S}; \mathfrak{A}, T) 
    := \sum_{t=1}^T R_t \mathbbm{1}[A_t \in \set{S}_i]$,
which directly corresponds to the canonical case that
a fixed portion of the rewards
is shared from the principal.
We note that
the entire analysis can be extended for a more general definition of utility
including discount factor over time and 
non-negative marginal utility, rigorously described in Appendix~\ref{sec:appendixprelim}.
A strategic agent $i$ 
aims at maximizing the {\em expected} utility, denoted by
\begin{align}
u_i(\set{S}; \mathfrak{A}, T) 
:= \EXP[ v_i(\set{S}; \mathfrak{A}, T) ] 
\;,
\end{align}
where the expectation takes over the randomness
of algorithm $\mathfrak{A}$ and reward $R_t$'s.
We define {\em dominant strategy} as follows:
\begin{definition}[Dominant strategy] \label{def:dominant}
Agent $i$'s strategy $\set{S}_i$ is 
dominant if 
regardless of the other agents' strategy $\set{S}_{-i} 
:= (\set{S}_1, \ldots, \set{S}_{i-1}, \set{S}_{i+1}, \ldots \set{S}_n)$, 
it provides 
the expected utility 
at least as much as any other $\set{S}'_i$ of support $\set{O}_i$ does, i.e., 
\begin{align} \label{eq:dominant}
u_i(\set{S}_i, \set{S}_{-i}) \geq u_i(\set{S}'_i, \set{S}_{-i}) \;,
\end{align}
where 
for simplicity,
we often write
$u_i(\set{S}_i, \set{S}_{-i})
= u_i(\set{S}_i, \set{S}_{-i}; \mathfrak{A}, T)$
omitting $\mathfrak{A}$ and $T$.
The dominance of $\set{S}_i$ is strict if the inequality
\eqref{eq:dominant} is strict.
\end{definition} 
It is clear that 
once an agent $i$ is able to identify a dominant strategy $\set{S}_i$, the agent's rational decision is always selecting $\set{S}_i$
regardless of the others' strategy $\set{S}_{-i}$.
Hence, it can be used to characterize the Nash equilibrium of agents' strategy.
However, depending on the principal's algorithm and the agent's knowledge,
it is possible that there exists no dominant one, or the agent has no ability to identify one.
In our analysis, we will elaborate 
the existence or identifiability of dominant strategy.
% ...

\subsection{Principal's objective}
\label{sec:principal}

The principal aims at maximizing the utility, 
which is defined by the accumulated reward from selecting arms submitted by the agents.
In other words,
the principal's objective can be alternatively formulated
as the minimization of the expected cumulative opportunity cost
defined as $\regret(\set{S}; \mathfrak{A}, T)$ in the followings:
\begin{definition}[Principal's regret]
Given strategy $\set{S}$, 
the regret of algorithm $\mathfrak{A}$ up to time horizon $T$ 
is 
\begin{align*}
\regret(\set{S}; \mathfrak{A}, T) := 
\sum_{t=1}^T \sum_{i \in \set{N}}\sum_{a \in \set{S}_i} \Delta(a) \EXP 
\big[ 
\mathbbm{1}[A_t = a]
\big] \;,
\end{align*}
where we let $\Delta(a) := \mu^\star - \mu(a)$ denote
the expectation of the instantaneous regret from selecting arm $a$
instead of the optimal arm $\mu^\star := \max_{a' \in \set{O}} \mu(a')$.
\end{definition}
It is well-known that without any prior knowledge,
the optimal regret is scaling with the number of arms \cite{mab:lowerbound}.
To be specific, if $\set{S}$ includes the best arm
$a^\star := \argmax_{a \in \set{O}} \mu(a)$, then
the optimal regret is given by $O(|\set{S}| \ln T)$,
where $|\set{S}|$ is the number of arms registered.
Such a scaling regret in the number of arms
is inevitable mainly because 
$O(\ln T)$ explorations per arm
is required to 
information-theoretically
identify the best arm $a^\star$ with an error probability less than $O({1}/{t})$ at round $t$
so that the regret from misidentifying $a^\star$ is bounded by $O\big(\sum_{t=1}^T {1}/{t} \big)=O(\ln T)$.
We refer to \cite{mab:lowerbound} for a rigorous analysis.
This illustrates the importance of maintaining the number of suboptimal arms as less as possible, while having smaller number of arms is also beneficial in terms of
principal's management cost, e.g., server capacity, in practice.

\section{Preliminary}
\label{sec:preliminary}
In this section, we provide an analysis revealing
the vulnerability of a canonical algorithm and a simple baseline 
against replicators. This shows not only the
significance but also the challenge of the replication problem. 


\subsection{Replication-proneness} 
\label{sec:sybil}

We first investigate the strategic behavior of agents
when using a canonical algorithm of regret minimization,
and check if the agents 
have a strong motivation to replicate 
arms under the algorithm, i.e., the algorithm is {\em replication-prone}.
\begin{definition}[Replication-proneness]
An algorithm $\mathfrak{A}$ for time horizon $T$ 
is replication-prone
if for an agent $i \in \set{N}$, any strategy
$\set{S} = (\set{S}_i, \set{S}_{-i})$, 
there exists $\set{S}'_i$ such that $\set{S}_i \subset \set{S}'_i$ and
\begin{align} 
u_i(\set{S}'_i, \set{S}_{-i}; \mathfrak{A}, T) >u_i(\set{S}_i, \set{S}_{-i}; \mathfrak{A}, T) \;. \label{eq:prone}
\end{align}
\end{definition}
In words, a replication-prone algorithm provides 
a strictly positive gain from replicating a part of arms for each agent,
and thus each agent's best decision 
is to register infinitely many arms regardless of the others' strategy.
Our analysis verifies that UCB1 \cite{mab:ucb1} is replication-prone, meanwhile 
we believe that an extended analysis can show
the replication-proneness of other canonical MAB algorithms such as KL-UCB \cite{mab:klucb1}.
\begin{theorem}\label{thm:negative}
UCB1 is replication-prone for $T \geq n$.
\end{theorem}
This theorem shows a theoretical understanding of replicating behaviors in practice,
and an importance to prevent replicators.
It is intuitive from the fact that 
UCB1 is oblivious to agents
and thus grants a certain amount of explorations for every single registered arm.
We provide a formal proof of
Theorem~\ref{thm:negative} in Appendix~\ref{sec:proofs}, 
where we mainly use a sophisticated coupling argument 
to compare rewards from two MAB processes
with or without
an additional replication of
the entire set of agent $i$'s original arms.
We note that our analysis requires no special 
condition on $\set{O}$ to show \eqref{eq:prone},
and thus 
just given UCB1, every agent $i$ is able to realize 
that its best decision must include infinitely many replications
even without knowledge on its arms or others. 

As mentioned earlier,
UCB1 with registration of suboptimal arms more than $\Omega(T/\ln T)$ may suffer from linear regret.
To overcome such an unwilling behavior of the algorithm, one might consider using the algorithms provided in many-armed bandit literature. However, as we briefly discussed in Section~\ref{sec:intro}, this may not work well when only some of the agents replicate their arm, and even the replication itself harms the principal since it exhausts the hardware resources.
Therefore, we instead aim at demotivating the agents replicating the arms in advance.

\subsection{Replication-proofness and Fair($\mathfrak{A}$)}
\label{sec:truthful}
To prevent the infinite replicas in advance, 
we focus on constructing a mechanism that makes
agents' Nash equilibrium at not replicating the arms.
To this end, 
we define {\em replication-proof} algorithm as a solution concept of our problem as follows:
\begin{definition}[Replication-proofness] \label{def:proof}
An algorithm $\mathfrak{A}$ is replication-proof if there exists a dominant strategy that no agent replicates.
\end{definition}
Due to the equality in 
\eqref{eq:dominant}
the replication-proofness actually refers to the weakly dominant strategy in the literature. Hence, it does not guarantee that making no replication is  making more money for the agents, which possibly induces a deviation for the agents. However, in practice, this can be prevented by introducing a tiny constant cost in registering the arms or charging a periodical commission fee with respect to the number of registered arms in the platform.
\begin{algorithm}
\SetAlgoLined
$i \in \set{N}$, MAB algorithm $\mathfrak{A}$\\
 \While{true}{
  Chooses agent $i \in \set{N}$
  uniformly at random
  \\
  Run $\mathfrak{A}$ with arms in $\set{S}_i$
 }
 \caption{Fair($\mathfrak{A}$)}
\end{algorithm}
In this sense, under a replication-proof algorithm,
the agents have a Nash equilibrium without replicating arms.
Replication-proofness and -proneness are not complementary to each other,
while one implies the negation of the other, i.e.,
UCB1 is replication-prone and thus it is not replication-proof.
We remark that Definition~\ref{def:proof} shares a similar virtue with the truthfulness in the mechanism design literature~\cite{agt}
since truthfully reporting a private valuation can be regarded as reporting only a part of the original arms without any replication in our scenario.



We note that
the replication-proofness can be
achieved by a very simple baseline equipped
with any MAB algorithm $\mathfrak{A}$, denoted 
by Fair($\mathfrak{A}$).
At each round, 
Fair($\mathfrak{A}$) chooses an agent $i \in \set{N}$
uniformly at random,
and then plays an arm $a \in \set{S}_i$
from running $\mathfrak{A}$ based on the history of the agent $i$'s arms.
Then it is straightforward to check that Fair($\mathfrak{A}$) is replication-proof since the selection of each agent is not affected by the other agents' strategy for any $\mathfrak{A}$. 
In addition, if possible, the best strategy of each agent is to register its best arm.
However, even with the best strategy containing only each agent's best arm,
every agent is selected with the equal probability at each round, and thus
any suboptimal agents will be selected in linear portion of rounds, i.e., linear regret is inevitable.

Therefore, 
the principal's objective can be given in two folds: 
i) demotivating the agents from replicating arms; and ii) 
acquiring small regret.
More formally,
we aim at designing an algorithm that is replication-proof and has sublinear regret at the same time.

\section{Main result}
\label{sec:result}
In what follows, we first describe 
an replication-proof algorithm, named Hierarchical UCB (H-UCB), which has a sublinear regret assuming every agent is fully rational and thus not replicating arms.
We then propose a robust version of H-UCB (RH-UCB), that are guaranteed to be replication-proof and of sublinear regret 
even in the realistic scenario possibly with irrational agents. 


\subsection{Hierarchical UCB}\label{sec:hucb}
We recall that Fair($\mathfrak{A}$) is equipped with the replication-proofness 
by the agent selection, irrelevant to the number of arms that the agent has,
in advance of the arm selection. Inspired by such a hierarchy of selections 
of agent and then arm, we devise hierarchical UCB (H-UCB) consisting of
two phases per round, 
where {\tt Phase 1} to choose an agent
is followed by {\tt Phase~2} 
to select its arm, sequentially.
We provide a pseudo code of H-UCB in Algorithm~\ref{alg:hucb},
where we keep track of the empirical mean rewards from each agent $i$ and arm $a$ by $R(i)$ and $r(a)$, respectively,
and the numbers of rounds selecting agent $i$'s arms and 
an arm $a$ by $N(i)$ and $n(a)$, respectively.
In {\tt Phase 1}, 
as a part of identifying an agent with the optimal arm,
a UCB algorithm selects an agent $\hat{i}$ of the highest value
of agent index $R(i) + \sqrt{\frac{2\ln t}{N(i)}}$ which estimates
the potential rewards from agent $i$'s arms at $t$-th round.
Then, in {\tt Phase 2},
given the empirically best agent $\hat{i}$,
we aim at selecting the agent $\hat{i}$'s best arm.
To this end, another UCB algorithm is employed to choose an arm $\hat{a}$ from the agent $\hat{i}$
of the highest value of arm index $r(a) + \sqrt{\frac{2\ln N(\hat{i})}{n(a)}}$
that assesses the potential of selecting arm $a$ 
on the part of time horizon when selecting the agent $\hat{i}$ only.
This forms a hierarchy of UCB algorithms
as the arm selection in {\tt Phase 2} depends on the agent selection in
{\tt Phase 1}.
We then have the following analysis on H-UCB:
\begin{theorem}\label{thm:H-UCB_superior}
H-UCB is replication-proof.
\end{theorem}
\begin{algorithm} 
\SetAlgoLined
Initialize $R(i), N(i), r(a)$ and $n(a)$ at zero  
for all $i \in \set{N}$ and $a\in \cup_{i \in \set{N}} \set{S}_i$;
 \For{$t=1,2,\ldots$}{
  \tcp{Phase 1 - agent selection}
  \eIf{Unexplored agent exists}{ 
    Pick an unexplored $\hat{i}\in \{i \in \set{N}: N(i) = 0\} $; 
  }{
    Pick $\hat{i} = \argmax_{i \in \set{N}} \left\{R(i) + \sqrt{\frac{2 \ln t}{N(i)}} \right\}$;
  }
  
  \tcp{Phase 2 - arm selection}
  \eIf{Unexplored arm exists in $\set{S}_{\hat{i}}$}{
    Pick an unexplored $\hat{a} \in \{a \in \set{S}_{\hat{i}}: n(a) = 0\} $; 
  }{
    Pick $\hat{a} = \argmax_{a \in \set{S}_{\hat{i}}} 
    \left\{r(a) + \sqrt{\frac{2 \ln N(i)}{n(a)}} \right\}$;
  }
  
  Play $A_t = \hat{a}$ and receive reward $R_t$; \\
  Update statistics:
  $r(\hat{a}) \leftarrow \tfrac{r(\hat{a})n(\hat{a}) + R_t}{n(\hat{a})+1}; ~~
  R(\hat{i}) \leftarrow \tfrac{R(\hat{i})N(\hat{i}) + R_t}{N(\hat{i})+1};$
  
  $n(\hat{a}) \leftarrow n(\hat{a}) + 1; ~~ N(\hat{i}) \leftarrow N(\hat{i}) + 1;$
 }
 \caption{Hierarchical UCB (H-UCB) \label{alg:hucb}}
\end{algorithm}

Intuitively, 
the exploration of the best agent in {\em phase 1}
distributes $O(\ln t)$ budget for each agent regardless of the number of the agent's arms.
In addition, given the asymptotically equal budget per agent,
if the agent replicated its arms, then it would increase
the risk of misindentifying its best arm and reduce its utility. 
Hence, this demotivates the agent from replicating arms.
A formal proof of Theorem~\ref{thm:H-UCB_superior} is provided in Appendix~\ref{sec:proofs}. In the proof, we in fact show  
a stronger notion of replication-proofness
based on stochastical orderings to compare
two different strategies instance-wisely.
Thus, H-UCB is replication-proof even when agent seeks 
gain from replication in randomness.

We remark that under H-UCB,
assuming each agent $i$ knows its best arm 
$a^\star_i \in \argmax_{a\in \set{O}_i} \mu(a)$,
the agent's dominant strategy is to register only $a^\star_i$'s.
On such a equilibrium, 
we obtain a logarithmic bound on the regret of H-UCB:
\begin{theorem}\label{thm:H-UCB_regret}
Given the original arms $\set{O}$ and the corresponding mean rewards $\{\mu(a): a \in \set{O}\}$, 
under an equilibrium $\set{S}$ when each agent $i$ decides 
its strategy knowing its best arm $a^\star_i$, 
$\regret(\set{S}; \textnormal{H-UCB}, T)$ is at most
\begin{align*}
    \sum_{i: \Delta_i > 0 } \frac{8}{\Delta_i}\ln T + 
    \left(1 + \frac{\pi^2}{3}\right) \left(\sum_{i=1}^n \Delta_i \right)\; ,
\end{align*}
where $\Delta_i := \mu^\star - \max_{a \in \set{O}_i} \mu(a)$.
\end{theorem}
The proof of Theorem~\ref{thm:H-UCB_regret}
is given in Appendix~\ref{sec:proofs}, where 
we present a formal characterization of equilibrium under the assumption.
We note that the asymptotic order of regret bound is $O(n \ln T)$,
which scales with the number of agents rather than
the total number of arms, while
in practice, it is much harder 
to replicate agents than arms
since a genuine identification, such as social security number or bank account, 
is required to register an agent.

\subsection{Robust Hierarchical UCB}
\label{sec:robust}

The logarithmic regret of H-UCB in Theorem~\ref{thm:H-UCB_regret}
assumes that every agent
is rational enough to understand no benefit from replicating arms under H-UCB,
and registers its arm truthfully.
Meanwhile, in practice, 
it is still possible to 
have {\em replicators} who falsely duplicate their arms even under a replication-proof
algorithm. 
We assume that each replicator $i$'s strategy $\set{S}_i$
includes at least a constant portion of arms of its best mean reward $\mu^\star_i:= \max_{a \in \set{O}_i} \mu (a)$, i.e.,
there exists a constant $c > 0$ such that
$\min_{i \in \set{N}} \frac{|\{a \in \set{S}_i: \mu^\star_i = \mu(a) \}|}{|\set{S}_i|} > c$.
This is a reasonable assumption since no replicator wants to hide 
its best arm. 
Even when a replicator cannot identify its best arm, the safest choice is to replicate the entire arms equally.


To obtain sublinear regret even with the presence of such replicators, 
modifying H-UCB, we propose robust H-UCB (RH-UCB), described in Algorithm~\ref{alg:rhucb}.
Before explaining the design of RH-UCB, we note that
RH-UCB inherits the hierarchical structure of H-UCB.
Hence, an analog analysis to that of Theorem~\ref{thm:H-UCB_superior}
concludes that:
\begin{theorem}\label{thm:RH-UCB:replicationfree}
RH-UCB is replication-proof.
\end{theorem}
We now explain rationales behind RH-UCB,
where we append a modification to each phase of H-UCB.
In {\tt Phase 2}, 
for each agent $i$, 
the risk of misidentifying the best arm in strategy $\set{S}_i$
increases with the size of $\set{S}_i$.
Hence, as a part of bounding the risk in {\tt Phase 2}, 
RH-UCB begins with
$\set{B}_i$'s from subsampling each $\set{S}_i$
uniformly at random 
without replacement
such that
the cardinality of subset $\set{B}_i$
is at most $L  \ln T$, where $L > 0$ is a hyperparameter.
The choice of size bound depends on $T$ for analytical simplicity,
while our analysis can be extended to an online subsampling with increasing size bound $L  \ln t$.
Given the subsampled $\set{B}_i$'s, in {\tt Phase 2}, 
we use the same UCB1 for $\set{B}_{\hat{i}}$.
Although the number of arms in $\set{B}_i$ has the upper bound $L  \ln T$,
but it is scaling in $T$.
Hence, the exploration budget $O(\ln T)$ distributed over agents 
in H-UCB may be insufficient to identify each agent's best arm at a desired confidence.
For this reason, we increase confidence interval of {\tt Phase 1}
from $\sqrt{\frac{\ln t}{N(i)}}$ to $\sqrt{\frac{\sqrt{t} \ln t}{N(i)}}$.

\begin{algorithm} 
\SetAlgoLined

Initialize $R(i), N(i), r(a)$ and $n(a)$ at zero  
for all $i \in \set{N}$ and $a\in \cup_{i \in \set{N}} \set{S}_i$;

Subsample $\set{B}_i$ of size $\min\{|\set{S}_i|, L \ln T \}$ from $\set{S}_i$;  \\
 \For{$t=1,2,\ldots$}{
  \tcp{Phase 1 - agent selection}
  \eIf{Unexplored agent exists}{ 
    Pick an unexplored $\hat{i}\in \{i \in \set{N}: N(i) = 0\} $; 
  }{
    Pick $\hat{i} = \argmax_{i \in \set{N}} \left\{R(i) + \sqrt{\frac{\sqrt{t}\ln t}{N(i)}} \right\}$;
  }
  
  \tcp{Phase 2 - arm selection}
  \eIf{Unexplored arm exists in $\set{B}_{\hat{i}}$}{
    Pick an unexplored $\hat{a} \in \{a \in \set{B}_{\hat{i}}: n(a) = 0\} $; 
  }{
    Pick $\hat{a} = \argmax_{a \in \set{B}_{\hat{i}}} 
    \left\{r(a) + \sqrt{\frac{2 \ln N(i)}{n(a)}} \right\}$;
  }
  
  Play $A_t = \hat{a}$ and receive reward $R_t$; \\
  Update statistics:
  $r(\hat{a}) \leftarrow \tfrac{r(\hat{a})n(\hat{a}) + R_t}{n(\hat{a})+1}; ~~
  R(\hat{i}) \leftarrow \tfrac{R(\hat{i})N(\hat{i}) + R_t}{N(\hat{i})+1};$
  $n(\hat{a}) \leftarrow n(\hat{a}) + 1; ~~ N(\hat{i}) \leftarrow N(\hat{i}) + 1;$
 }
 \caption{Robust H-UCB (RH-UCB)\label{alg:rhucb}}
\end{algorithm}

For ease of exposition, we assume that there exists a unique agent $i^\star$ who possesses the optimal arm\footnote{Indeed, multiple optimal agents only reduce the regret bound, since the constant $C$  in our regret bound $O(C\sqrt{T}\ln T)$ will decrease and $O(\ln^2T)$ appear instead.}, and present the following regret bound for RH-UCB:
\begin{theorem}\label{thm:RH-UCB:regret}
Given the original arms $\set{O}$ and $\{\mu(a): a \in \set{O}\}$, under any $\set{S}$ given that all the agents are either strategic or replicator, if $L \geq 1/c$, then $\regret(\set{S}; \textnormal{RH-UCB}, T)$ is at most
\begin{align*}
\sum_{i \in \set{N}}
\sqrt{T}\ln T 
\bigl(
1 + \frac{4}{\Delta^2_i}  + 
\sum_{\substack{a\in \set{O}_{i^\star} \\ \delta_{i^\star,a}>0}} \frac{50L}{\delta_{i^\star, a}} 
&+ 
\sum_{\substack{a\in \set{O}_{i} \\ \delta_{i,a}>0}} \frac{114L}{\delta_{i, a}}
\bigr)+ O(\ln^5 T)\;,\nonumber
\end{align*}
where $\delta_{i,a} =
\mu^\star_i - \mu(a)$.
\end{theorem}


In Appendix~\ref{sec:proofs}, we provide not only the proof of Theorem~\ref{thm:RH-UCB:regret}
but also a further discussion on each term in the regret bound.
Even with the replicators,
the regret of RH-UCB is asymptotically $O(\sqrt{T} \ln T)$
for any given configuration of the original arms.
Hence, RH-UCB is replication-proof but also achieves sublinear. 
The majority of regret $O(\sqrt{T} \ln T)$ is from the probability 
of misidentifying each agent $i$'s best arm among subsampled $\set{B}_i$
of size $O(\ln T)$.
We note that $O(\sqrt{T}\ln T)$ amount of agent exploration indeed optimizes the order of regret bound in our analysis.
For any $1$-dimensional increasing function $f$, suppose that {\tt Phase 1}'s agent exploration term is given by $R_i + \sqrt{\frac{f(t)}{N(i)}}$. Then, it can similarly be derived that regret bound becomes $O(f(T)) + O(\frac{T\ln^2T}{f(T)})$, where the optimized regret can be obtained for $f(t) = O(\sqrt{t}\ln t)$ due to the inequality of arithmetic and geometric means.

We note that the regret analysis in Theorem~\ref{thm:RH-UCB:regret}
is problem-dependent as the bound is given as a function of the configuration of original arms. In the following theorem, we further present
a problem-independent regret bound of RH-UCB:
\begin{theorem}
RH-UCB has problem-independent regret upper-bound $O(\frac{n}{\sqrt{c}}T^{\frac{3}{4}}\ln T)$ for sufficiently large~$T$.
\end{theorem}
The detailed terms and proof are presented in Appendix~\ref{sec:proofs},
where we provide a worst-case analysis on regret. We note that
with no replication, 
the optimal asymptotic of problem-independent regret is known to be $O(\sqrt{T})$
\cite{mab:audibert}, while 
RH-UCB has $O(T^{3/4}\ln T)$ even with replicators. 
Hence, the gap between $O(\sqrt{T})$ and $O(T^{3/4}\ln T)$
can be interpreted as an algorithmic cost for being robust and replication-proof.

Note that our regret bound holds for $L \geq 1/c$.
Running RH-UCB with sufficiently large $L$ would guarantee this bound, but it might hurt the algorithm's efficiency since there could be a chance that too much redundant arms are subsampled for each agent. In case when precisely estimating $c$ a priori is difficult, we can instead subsample $\ln^2 T$ arms per agent since sufficiently large $T$ would guarantee $\ln T \geq 1/c$.
Meanwhile, this still requires the principal to determine $T$ in advance. This dependency can be removed by gradually subsampling arms up to $\ln^2 T$ within each agent.
These techniques indeed enable us to derive $O(\frac{n}{c^2}\sqrt{T\ln^3 T})$ regret and $O(\frac{n}{\sqrt{c}}T^{\frac{3}{4}}\ln^{\frac{3}{2}}T)$ problem-independent regret without any prior information on $L$ and $T$, where we put back a detailed description of the algorithm and corresponding analysis in Appendix~\ref{sec:prh-ucb}.

We remark that both the hierarchical structure and the subsampling are necessary to ensure sublinear regret. If there exists no hierarchy, it might work poor when only the suboptimal agents are replicators, but optimal one is not.
Otherwise if subsampling is discarded, any replicators would hurt the system since their optimal arms cannot be identified.
We provide a relevant ablation study in Section~\ref{sec:simulation}.



Finally, it is obvious that H-UCB works better when all the agents only register their optimal arm, since RH-UCB achieves $O(\sqrt{T}\ln T)$-regret in this case.
However, we need to put more exploration in {\tt Phase 1} to ensure that all the arms within the agent are explored sufficiently under the existence of replicators.
Hence, by closely observing behavior of the agents in the system, one may balance the trade-off between robustness and efficency.
We will present a relevant case study in the following section.

\section{Simulation}
\label{sec:simulation}

\begin{figure}[!t]
\centering
\captionsetup{justification=centering}
\subfloat[Reward of $0.5$-agent.]{
\includegraphics[width=0.38\columnwidth]{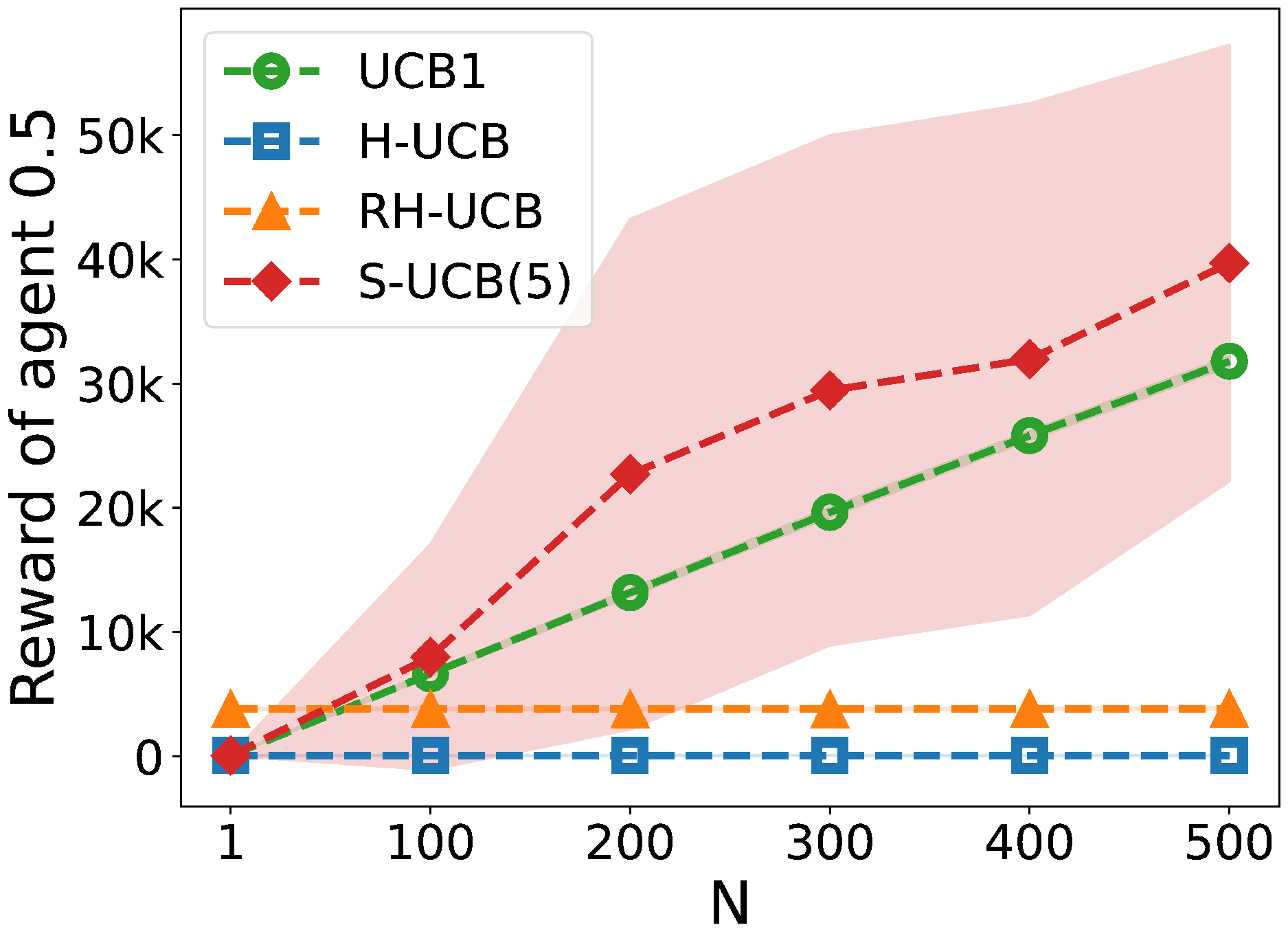}
\label{fig:reward_low}
}
\hspace{0.4cm}
\subfloat[Reward of $0.9$-agent.]{
\includegraphics[width=0.38\columnwidth]{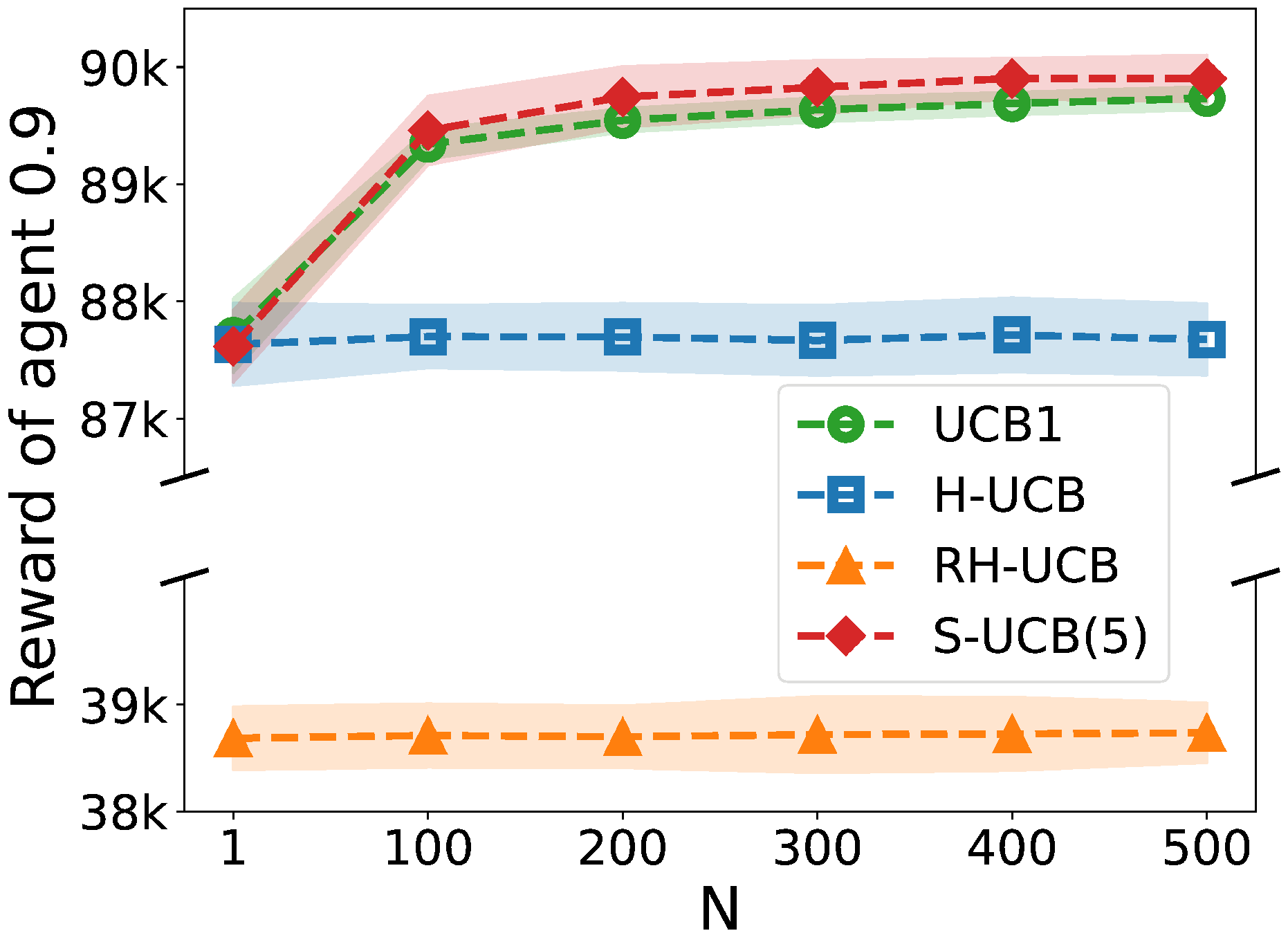}
\label{fig:reward_high}
}
\captionsetup{justification=justified}
\caption{(a) The average revenue and one-sigma interval for $0.5$-agent as it replicates more arms in the single original arm setup.
(b) The revenue of $0.9$-agent in the same setup with (a) but when $0.9$-agent replicates.
}
\label{fig:sim1}
\vspace{-0.2cm}
\end{figure}

We evaluate the proposed algorithms and compare them with existing ones by borrowing well-known open-source library SMPyBandits~\cite{SMPyBandits}.
In addition to UCB1, H-UCB and RH-UCB, we measure the performance of S-UCB($l$) which samples at most $l\ln T$ arms among the entire arms given $T$, and then runs UCB1 on them. That is, S-UCB($l$) is an algorithm where the hierarchy is removed from RH-UCB, discarding the agent selection phase. We decide the subsampling ratio $l$ as the number of total original arms.
Each solid or dashed line indicates the mean value of the treatment over the repetitions, and its shade refers to the one-sigma interval.

Figure~\ref{fig:sim1}\subref{fig:reward_low} and~\ref{fig:sim1}\subref{fig:reward_high} represent the changes of agents' revenue with respect to their strategies in each algorithm.
In both scenarios, we consider $5$ agents each of whom has a single original Bernoulli arm with parameters $0.5, 0.6, 0.7, 0.8, 0.9$, respectively. We denote $x$-agent as the agent whose optimal arm's parameter is $x$. Figure~\ref{fig:sim1}\subref{fig:reward_low} refers to $0.5$-agent's revenue when only the $0.5$-agent replicates its arm among the agents, and Figure~\ref{fig:sim1}\subref{fig:reward_high} indicates that of $0.9$-agent when only the $0.9$-agent replicates.

\paragraph{Vulnerability of non-hierarchical algorithms}
The replicating agents' mean revenue increases in both scenarios under UCB1 and S-UCB, but stays consistent under H-UCB and RH-UCB. The revenue increment of $0.5$-agent is more drastic than that of $0.9$-agent. It is because the replication of the optimal arm does not fundamentally change the asymptotic selection ratio of any sub-optimal arms, but only delays their selections. Note that the $0.5$-agent's revenue increases almost in linear manner which is quite intuitive regarding that every suboptimal arm is sampled $O(\ln T)$ times in UCB1. On the other hand, in S-UCB, the revenue increases in somewhat convex manner due to the nature of probabilistic subsampling.

\begin{figure*}[!t]
\centering
\captionsetup{justification=justified}
\subfloat[Cumulative regret when $0.5$-agent  is a replicator.]{\includegraphics[width=0.3\columnwidth]{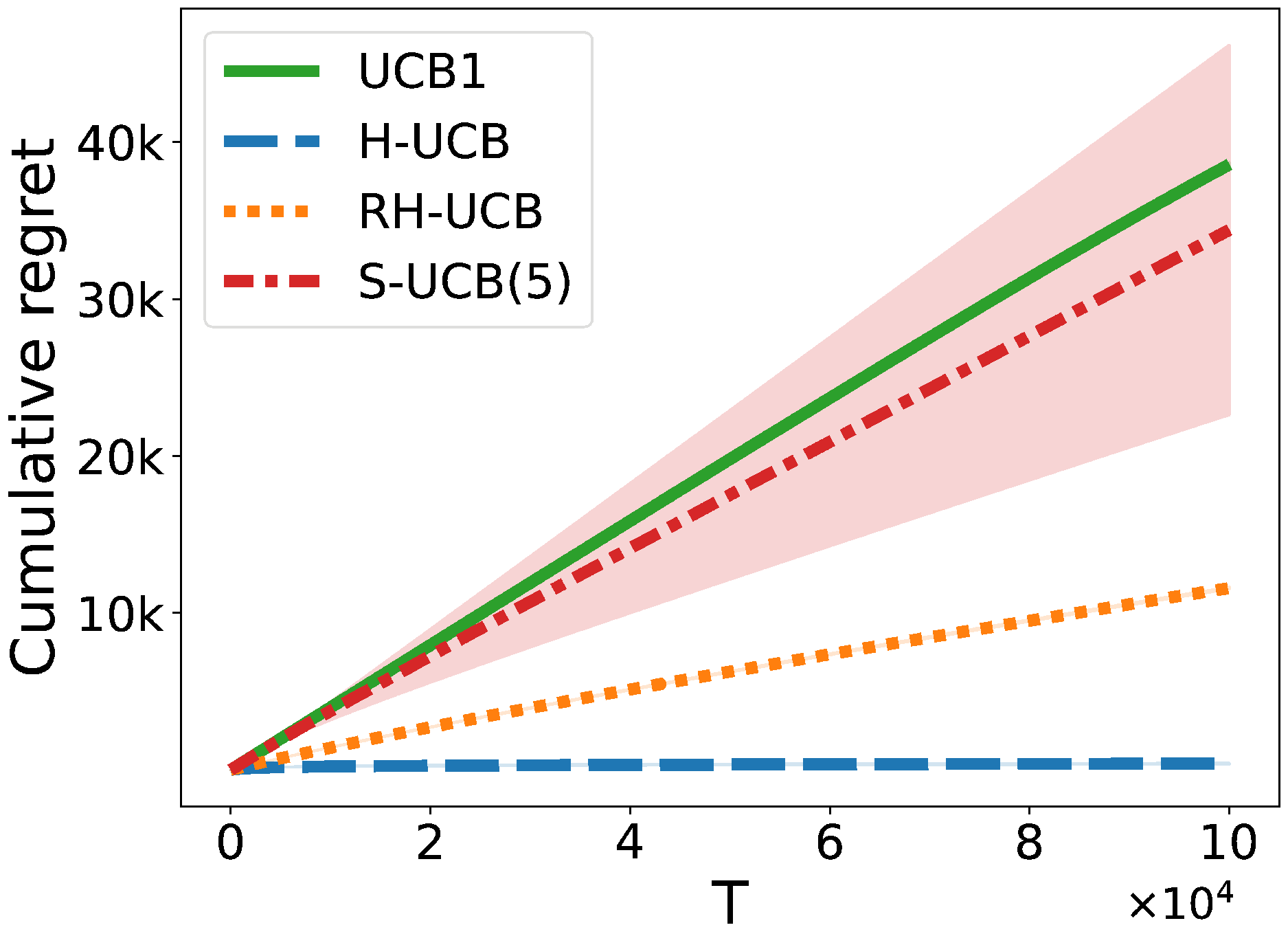}
\label{fig:regret_poor}
}
\hspace{0.1cm}
\subfloat[Cumulative regret when all agents except $0.9$-agent are replicators.]{
\includegraphics[width=0.3\columnwidth]{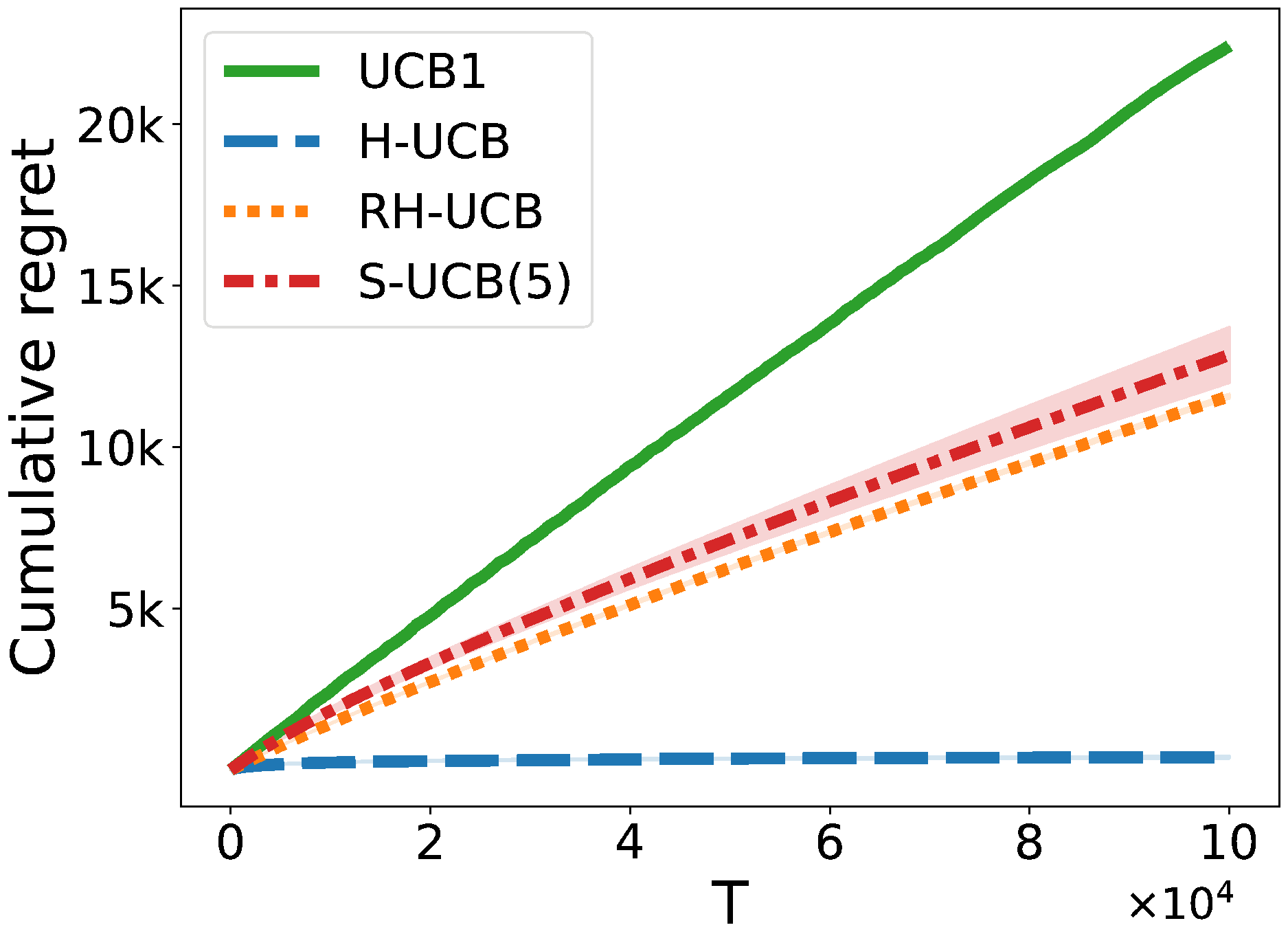}
\label{fig:regret_many}
}
\hspace{0.1cm}
\subfloat[Cumulative regret when $0.5$, $0.6$, $0.7$-agents are replicators, and $0.8$, $0.9$-agents partially replicate.]{
\includegraphics[width=0.3\columnwidth]{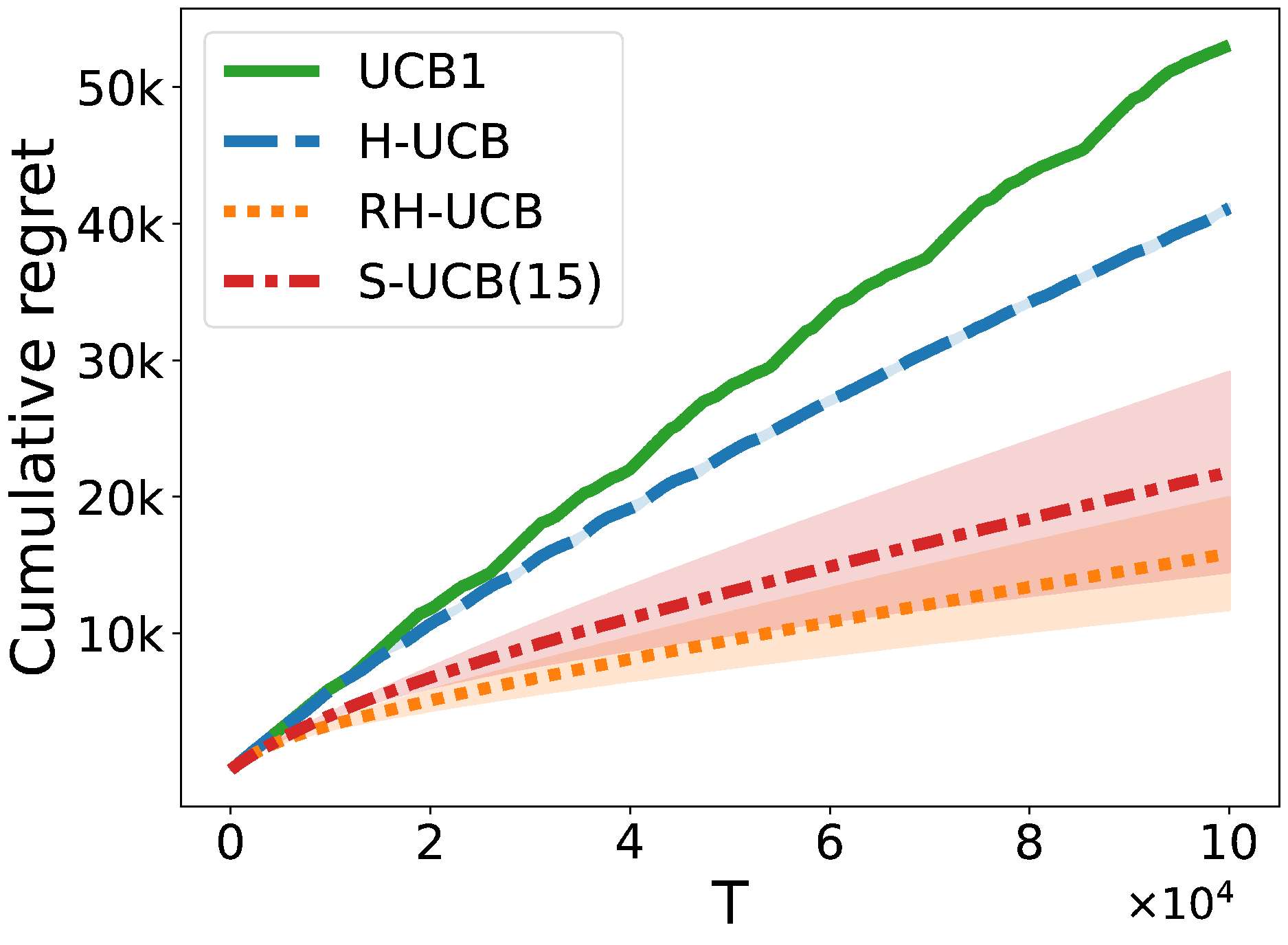}
\label{fig:regret_opt}
}
\caption{(a) The average cumulative regret and one-sigma interval for the principal in each algorithm, when $0.5$-agent replicates its arm $1000$ times in the single original arm setup.
(b) This shares the same setup with (a), but when all the agents except $0.9$-agent replicate their arms $1000$ times.
(c) $0.5$, $0.6$, $0.7$-agents replicate their original arms $1000$ times and $0.8$, $0.9$-agents partially replicate their arms in the multiple original arms setup.
}
\vspace{-0.4cm}
\label{fig:sim2}
\end{figure*}

\paragraph{Hierarchy makes low regret}
Figure~\ref{fig:sim2}\subref{fig:regret_poor} and Figure~\ref{fig:sim2}\subref{fig:regret_many} depicts the principal’s cumulative regret in the single original Bernoulli arm setup with $5$ agents like Figure~\ref{fig:sim1}.
Figure~\ref{fig:sim2}\subref{fig:regret_poor} shows the regret when the $0.5$-agent replicates its arm $1000$ times but the others do not replicate at all. UCB1 and S-UCB suffer from the replication even if there is only one replicator agent. However, the hierarchy-based ones allocate a fixed portion of exploration within each agent. Because that amount of exploration is shared across the agent's arms, the replication does not significantly harm the principal. Since RH-UCB enforces a larger amount of agent exploration, low-reward agents are also explored more than H-UCB. This results in worse performance for RH-UCB than H-UCB.

Besides, in Figure~\ref{fig:sim2}\subref{fig:regret_many}, we consider a similar scenario but where $0.5$, $0.6$, $0.7$, $0.8$-agents replicate their arms 1000 times yet only the $0.9$-agent does not. H-UCB still achieves lower regret than RH-UCB since the optimal agent does not replicate arms at all. Hence we believe that H-UCB might work well if there exists any optimal agent who does not replicate its arm at all. This indeed can be observed in our analysis, where the predominant term $O(L\sqrt{T}\ln T)$ can be removed in the regret bound in Theorem~\ref{thm:RH-UCB:regret} so that only $O(\sqrt{T}\ln T)$ remains.

\paragraph{Primacy of RH-UCB in a world of replicators}
On the other hand, Figure~\ref{fig:sim2}\subref{fig:regret_opt} represents the regime that RH-UCB works more robustly.
We assume that there are $5$ agents with $3$ Bernoulli original arms. Each of those agents has arms with mean $0.2$ and $0.1$ in common, and has a different optimal arm whose parameter is $0.5$, $0.6$, $0.7$, $0.8$, or $0.9$, respectively.
The $0.5$, $0.6$, and $0.7$-agents are replicators and they register $1000$ arms per each original arm.
Besides, we assume that the $0.8$-agent and the $0.9$-agent are {\em partial replicators}. Each of those two agents registers its own optimal arm $10$ times, and the suboptimal arms $100$ times each.
In this case, the power of RH-UCB reveals so that it succeeds to identify the best agent and the corresponding best arm. However, the regret of H-UCB grows almost in linear manner, which is even worse than S-UCB. This is because H-UCB only explores each agent $O(\ln T)$ times, and the existence of many arms in the optimal agent possibly hinders H-UCB to discover the best arm within the agent. This would be the case when all the agents are replicators.


Our intention of designing hierarchically-structured algorithms is solely to demotivate the agents from replicating their arms.
Interestingly, however, our algorithms outperform existing ones in the context of cumulative regret in many regimes.
This might arise because there is an inherent dominance structure among the agents in the setups of our experiments including that of Figure~\ref{fig:sim2}\subref{fig:regret_opt}.
In the setups, the suboptimal arms remain the same across the agents, and the agents are only distinguished by the optimal arm of each agent so that the parameter of the optimal arm fully determines an agent-wise dominance.
Our hierarchical algorithms naturally exploit such inherent dominance structure among the registered arms by grouping them by agent, which possibly increases their performance.
Hence, we believe that H-UCB or RH-UCB would perform better than our theoretical guarantees since many practical scenarios may possess the structure across the agents.

\section{Conclusion}
\label{sec:conclusion}
We analytically study multi-armed bandit algorithms against strategic replication. 
We firstly show that UCB1 admits an infinitely many registration of arms for strategic agents which inevitably induces linear regret.
We mainly observe that it is essential to limit the amount of exploration per agent to demotivate such strategic replication, and present that a simple fair algorithm achieves it while it suffers from linear regret.
We propose H-UCB to catch the both rabbits of sublinear regret and replication-proofness, and present more robust algorithm named RH-UCB which adapts to the infinitely many replicas from replicators.
We conduct numerical experiments to validate our theoretical findings, and provide some practical insights from it.




\clearpage

\bibliography{ref}

\clearpage
\onecolumn
\appendix
\appendixpage


\section{Open problems and future directions}
\label{sec:future}
We discuss some open questions and future research directions in this section.

{\bf Tightening regret bound.}
We show that there exists an algorithm 
that is replication-proof and also has a sublinear regret
regardless of whether agents are strategic or not.
However, it is not studied whether the regret bound we obtain is tight or not.
To tighten the regret analysis, 
an fundamental limit analysis is necessary.
To be specific, we need a regret lower bound
for uniformly good and also replication-proof algorithms.
The fundamental analysis with the notion of replication-proofness
seems challenging.
Hence, we would bypass it
by studying a lower bound for uniformly good algorithms
given the scenario of replicator in our paper
with an argument of change-of-measure, e.g., \cite{mabstructure:combes}.
The relaxed scope of algorithms of interest
would provide a loose lower bound
to be achieved by replication-proof algorithms.
Meanwhile, this can provide a useful insight to improve
the regret upper bound fully exploiting 
the assumption on replicators, i.e., $\min_{i \in \set{N}} \frac{|\{a \in \set{S}_i: \mu^\star_i = \mu(a) \}|}{|\set{S}_i|} > c > 0$,
as in \cite{mabstructure:combes}, the lower bound analysis 
corresponds to an optimal upper bound algorithm.
In fact, the assumption is analog to the bandit problems with
linear or row-rank structures \cite{mab:auer:linear, mab:jun, mab:katariya, mab:kveton}.
However, our bandit problem with replicators is more challenging
since we do not know an embedding (or feature map)
of arms, which are provided in 
the linear and row-rank bandit problems.

{\bf Agent information model.}
We assume that the strategic agents 
have ability to {\em exactly} assess their arms, or at least 
to identify the best one. 
This may be not true in practice.
Hence, one can consider a Bayesian information model where each agent only knows (or believes) the meta distribution of its arm's parameter, but does not know the arm's distribution exactly.
This can be interpreted as a scenario of recommender system
to a population of heterogeneous clients.
Perhaps this meta distribution could be different for each arms within the agent.
In these cases, the agents might be motivated to register all the original arms since it will suffer from linear regret with constant probability if some of the original arms are not registered.
This is somewhat different from our result, since our model encourages the agent to exactly register only the best arm if such information is readily available. 
This difference seems natural since any agent who is oblivious to its own arms would definitely register all the arms if there exists no significant cost in registering more arms.
Nonetheless, in both cases, we believe that our hierarchical algorithms would also demotivate the agents from replicating their arms since replication might not result in an increase of the expected value of the empirical average reward.

{\bf Beyond UCB-based algorithms.}
We mainly reveal that an algorithm needs to restrict its exploration cost for each agent to achieve replication-proofness. For $\varepsilon$-greedy based algorithms~\cite{mab:ucb1}, it is straightforward to examine the exploration cost since the exploration phase and exploitation phase are explicitly separated. Hence, we believe that the hierarchical structure can analogously be applied to construct a replication-proof and sublinear regret $\varepsilon$-greedy based algorithm.
Meanwhile, it becomes nontrivial for algorithms based on probability matching~\cite{mab:ts1, mab:ts2}.
The exploration and exploitation phase are not clearly separated in these algorithms, hence we need to define the exploration cost in rigorous manner so that restricting exploration cost would result in a replication-proofness.
These analysis might be more challenging since we need to perform Bayesian analysis, and even these algorithms might be discouraged since they need information on a parametric structure of reward distributions.
However, we believe that there may exist some valuable insights in these analysis, since different notion of exploration cost could lead to a different point of the equilibrium. 
This might bridge the gap between our notion of replication-proofness and replication-proneness. Indeed, if we construct a logarithmic regret algorithm that is not replication-prone, but only admits a finite number of replications to the strategic agents, then it could be better than RH-UCB for the principal since the finite number of replications only contributes to increase constant term in the regret bound while RH-UCB admits square root regret. To fairly compare the various algorithms in this context, it might require us to define a new criterion to estimate their efficiency, e.g., worst-case regret over the dominant strategies (or over any strategies if no such one exists).

{\bf Beyond stochastic bandit with one-shot game.}
In practice, many recommendation systems are modeled as contextual bandit 
which utilizes features of clients and arms.
Since replicating could still hurt the platform under the contextual bandit algorithm, we believe that we can generalize our results into this setting.
Besides, we can consider a sequential arrival of the agents, which will result in a repeated-game between the principal and the agents.
In this case, the platform might have more controllability on handling agent's strategic behavior since agents are forced to register only certain amount of contents at once, and principal can block the agent's registration in advance at some rounds if they've been registered too much till that time.
On the other hand, agents could also abuse the system more viciously, for example, they can get rid of their registered contents, and register it again as if it is newly created ones, which might hamper the algorithm from identifying the optimal arm.
Finally we remark that our analysis on replication-proofness are mainly based on the fact that a family of Bernoulli distributions equipped with stochastic ordering form a totally ordered set. If such ordering cannot be made within a family of distributions in consideration, then we cannot compare utilities between two strategies with use of stochastic dominance. In this case, there might be a chance that such stochastic dominance between two strategies' utilities may not hold for any non-decreasing function $U$. Nevertheless, we believe that we can still compare the expected utilities, but it may require more sophisticated method to analyze on it.

\section{Preliminaries to The Analysis}
\label{sec:appendixprelim}
Our result on replication-proneness and replication-proofness can be shown in more generalized statements than the ones provided in the main paper.
In this context, before presenting the main proofs, we provide a generalized model for our problem, then introduce some preliminaries required in presenting our results.
We mainly consider two extensions: discount sequence and generalized utility function.
Throughout the appendix, we abuse $I_t(\set{S}_i, \set{S}_{-i})$ or $I_t$ to denote the arm selected at round $t$ instead of $I_t(\set{S}_i, \set{S}_{-i}; \mathfrak{A}, T)$ if the context is clear.

\subsection{Discount sequence}
As discussed, we consider a scenario that each agent $i$ receives the sum of rewards from its arms.
In practice, this corresponds to the scenario where a fixed portion of the rewards may be shared from the principal to the agents.
Then, agent's revenue can precisely be defined as $v_i(\set{S}_i, \set{S}_{-i}) \cdot \varepsilon$. However, we can ignore this constant since we're only interested in comparing the agent's utility between two or more strategies, where the impact of $\varepsilon$ will be wiped out in this case.
Hence, we re-define the following notion of agent's revenue by introducing a discount sequence:
\begin{definition}[Agent's discounted revenue]
Given a principal's algorithm $\mathfrak{A}$, time horizon $T \geq n$, other agents' strategic decision $\set{S}_{-i} = (\set{S}_1, \ldots, \set{S}_{i-1}, \set{S}_{i+1}, \ldots \set{S}_n)$, and discount factor $\gamma = (\gamma_t )_{t=1}^T$,  agent $i$'s revenue is defined as,
\begin{align*}
v_i(\set{S}_i, \set{S}_{-i};\mathfrak{A}, T) = \sum_{t=1}^T \gamma_t R_t\;.
\end{align*}
\end{definition}

The discount sequence captures the agent's various viewpoint in computing long-term cumulative reward, e.g. agents' utility can be discounted over time ($\gamma_t = 1/t, t\in \mathbb{Z}_{\geq0}$) or not ($\gamma_t = 1, t\in \mathbb{Z}_{\geq0}$). This generalized the practical scenario where the advertisers or content providers in online platform are usually myopic than the platform~\cite{repeated:amin, repeated:mohri}.
We define a discount sequence to be {\em proper} if it is non-negative real-valued ($\gamma_t \in \mathbb{R}_{\geq 0}$), non-increasing ($\gamma_t \geq \gamma_{t+1}$), and has at least $n$ non-zero elements.
We denote the family of all proper discount sequences given $T$ be $\Gamma_T$.

\subsection{Agent's expected utility}
Agent's revenue is actually a random variable which can be varying with respect to the randomness of the policy, and the randomness of realized reward.
Hence, there could be difference between the agents in how they assess the risk in the random revenue.
For example, some agents might want to avoid a large variability in the random revenue, but some might love it.
In this context, we consider the following general notion of agent's expected utility, instead of agent's expected revenue.
\begin{definition}[Agent's expected utility]
Given a utility function $U:\mathbb{R}_{\geq 0} \mapsto \mathbb{R}_{\geq 0}$, agent $i$'s expected utility is defined as,
\begin{align*}
u_i(\set{S}_i, \set{S}_{-i}; \mathfrak{A}, T) = \EXP [U(v_i(\set{S}_i, \set{S}_{-i};\mathfrak{A}, T))] \; .
\end{align*}
\end{definition}
Utility function $U$ represents how the agents assess an uncertainty in revenue.
If $U$ is concave, they gradually underestimate the revenue growth as its absolute amount becomes large, hence in this case they prefer a certain amount of revenue rather than taking a risk to get larger revenue even if the expected revenue is smaller. We say that the agents are {\em risk-averse} in this case.
In comparison, we call the agents be {\em risk-seeking} if $U$ is convex, since they are willing to take a risk for having a large amount of revenue with small probability.
If $U$ is an identity function, we say that the agents are {\em risk-neutral} and they only care about the expected amount of revenue.
In this case, expected utility simply reduces to the expected revenue which was provided in our main paper:
\begin{align*}
\sum_{t=1}^T \gamma_t  \sum_{a \in \set{S}_i}  \mu(a) \cdot \EXP \big[ \mathbbm{1}[I_t(\set{S}_i, \set{S}_{-i};\mathfrak{A}, T) = a]\big]\; .
\end{align*}
Although we consider a homogeneous scenario so that every agents share a common utility function, all the results can easily be extended to heterogeneous utility scenario.




\subsection{Further notions and lemmas}
Before getting into the proofs, we introduce some notions to make the explanation more clear.
We denote a pair of multi-armed bandit problem instance and corresponding algorithm be $\set{X} = (\langle T, \set{S}\rangle, \mathfrak{A})$ for time horizon $T$, set of arms $\set{S}$ (with corresponding set of reward distributions $d(\cdot)$), and MAB algorithm $\mathfrak{A}$. We call such $\set{X}$ be {\em MAB process}. 
We often omit $T$ in $\set{X}$ as $\set{X} = (\langle \cdot, \set{S}\rangle, \mathfrak{A})$ if $\mathfrak{A}$ is horizon-independent.
Given a MAB process $\set{X}$, its sample path $r^{\set{X}}_T = (i_t,x_t)_{t\in [T]}$, and any set of arms $\set{A} \subset \set{S}$, we define $\set{A}$-sample path $c_{\set{A}}(r^{\set{X}}_T)$ be $r^{\set{X}}_T$'s longest subsequence that consists of elements contained in $\set{A}$, i.e. 
\begin{align*}
    c_{\set{A}}(r^{\set{X}}_T) =\argmax_{r} \{|r|: r = (i_{n_t},x_{n_t})_{t\in [T]},1\leq n_1\leq n_2\leq \ldots \leq n_T \leq T,i_{n_t}\in\set{A} \} \; .
\end{align*}
We define {\em $t$-count} $N_{\set{A},t}^{\set{X}}$ be the number of times arm $i \in \set{A}$ is pulled in first $t$ rounds in $\set{X}$. We often omit $\set{X}$ in $N_{\set{A},t}^{\set{X}}$ if no confusion arises.
We abuse the notation as $N_{a,t}^{\set{X}}$ when we consider a single arm $a$ instead of a set of arms.

Next, we define a notion of stochastic dominance between two random variables as the following:
\begin{definition}[Stochastic dominance]
Let $X$ and $Y$ be two random variables.
$X$ has stochastic dominance over $Y$
and we denote $X \succeq Y$ if 
$\forall x\in \mathbb{R}$, $\prob{X \leq x} \leq \prob{Y \leq x}$.
If $X \succeq Y$ and there exists some $x$ such that $\prob{X \leq x} < \prob{Y \leq x}$, then we say that $X$ has strict stochastic dominance over $Y$, $X \succ Y$.
\end{definition}
We note that stochastic dominance always induce a higher expectation since $\EXP [X] = \sum_{x=0}^\infty \prob{X > x} \geq \sum_{x=0}^\infty \prob{X > x} = \EXP [Y]$, i.e. it is stronger ordering compared to ordering with expectation.
Our definition of stochastic dominance refers to the first-order stochastic dominance in the literature, and it enables us to analyze the behavior of agents regardless of their behavior against the risk.

We now introduce some well-known properties on stochastic dominance which are useful in proving our main results.
The following proposition provides a connection between stochastic dominance and expected utility:
\begin{proposition}\label{thm:dominance_to_utility}
Given random variables $X$ and $Y$, the following holds:
\begin{compactenum}[1)]
\item $X \succeq Y$ if and only if $\EXP [U(X)] \geq \EXP[U(Y)]$ for any non-decreasing utility function $U(\cdot)$,
\item $X \succ Y$ if and only if $\EXP [U(X)] > \EXP[U(Y)]$ for any strictly increasing utility function $U(\cdot)$.
\end{compactenum}
\end{proposition}
\begin{proof}
We only prove for the discrete case since it can easily be generalized.
Suppose that $X$ and $Y$ are discrete random variables with outcome $\{x_1,x_2, \ldots, x_k\}$ where $x_1 < x_2 < \ldots < x_k$.

Let $d_i = \PROB[Y \leq x_i] - \PROB[X \leq x_i]$ and we have
\begin{align*}
\EXP [U(X)] - \EXP [U(Y)] 
&= \sum_{i=1}^{k} (\PROB[X = x_i] - \PROB[Y = x_i])U(x_i) \nonumber
\\
&= \sum_{i=1}^{k-1}(\PROB[X\leq x_i] - \PROB[Y \leq x_i])(U(x_i) - U(x_{i+1})) + (\PROB[X \leq x_k] - \PROB[Y \leq x_k])u_k
\\
&= \sum_{i=1}^{k-1}(U(x_{i+1}) - U(x_{i}))d_i\; ,
\end{align*}
where the second equation holds from Abel's lemma and the last comes from $d_k = 0$.
It is straightforward that if $d_i \geq 0 $ for $\forall i=1,2,\ldots, k$ then $\EXP[U(X)] \geq \EXP[U(Y)]$ since $U$ is non-decreasing.
Suppose that $\EXP[U(X)] \geq \EXP[U(Y)]$ for any non-decreasing function $U$ and there exists $j$ such that $d_j < 0$. If we consider a utility function $U$ such that $U(x_{i+1}) = U(x_i)$ for $i \neq j$, and $U(x_{j+1}) > U(x_{j})$, then $\EXP [U(X)] - \EXP [U(Y)] = d_j (U(x_{j+1}) - U(x_j)) <0 $ which is a contradiction.
We omit the proof for the second statement since it can easily be derived in similar argument.
\end{proof}

We also note that the stochastic dominance is preserved under summation and multiplication between positive independent random variables:
\begin{proposition}\label{thm:dominance_preserve}
$X_1,X_2,Y_1$ and $Y_2$ are positive random variables such that $X_1$ and $Y_1$ is independent, and $X_2$ and $Y_2$ is independent.
If $X_1 \succeq Y_1$ and $X_2 \succeq Y_2$, then, (i) 
$X_1+X_2 \succeq Y_1+Y_2$ and (ii) $X_1X_2 \succeq Y_1Y_2$.
\end{proposition}
\begin{proof}
Since stochastic dominance is equivalent to the existence of monotone coupling~\cite{book:roch2015modern}, we can construct monotone couplings $(\hat{X}_1, \hat{Y}_1)$ and $(\hat{X}_2, \hat{Y}_2)$ of $(X_1, Y_1)$ and $(X_2, Y_2)$, respectively.
Then, we have $X_1+X_2 \sim \hat{X}_1 + \hat{X}_2 \succeq \hat{Y}_1 + \hat{Y}_2 \sim Y_1+Y_2$ and $X_1X_2 \sim \hat{X}_1\hat{X}_2 \succeq \hat{Y}_1\hat{Y}_2 \sim Y_1Y_2$ where $\sim$ refers to the equality in distribution, and it concludes the proof.
\end{proof}

Finally, with use of stochastic dominance, the following theorem represents a sufficient condition for a strategy to be dominant over another strategy for any non-decreasing utility function $U$.
We remark that we often say that a strategic is dominant over another strategy when it induces an expected utility which is at least that of another.
\begin{claim}\label{thm:dominant_general}
Given a principal's algorithm $\mathfrak{A}$, time horizon $T \geq n$, and other agents' strategic decision $\set{S}_{-i}$, suppose that agent $i$'s two strategies $\set{S}_i$ and $\set{S}'_i$ satisfy the followings:

\begin{compactenum}[1)]
\item $a' \succeq a$ for $\forall a' \in \set{S}'_i$ and $\forall a \in \set{S}_i$,
\item $N^{\set{X}'}_{\set{S}'_i,t}$ has stochastic dominance over $N^{\set{X}}_{\set{S}_i, t}$ for any $t \in [T]$
\end{compactenum}
Then, $\set{S}'_i$ is a dominant over $\set{S}_i$ under any non-decreasing utility function $U$ and any proper discount sequence $\gamma \in \Gamma_T$.
\end{claim}
\begin{proof}
We denote $r(a)$ be the revenue R.V. of arm $a$. 
The probability that agent $i$'s revenue is at most $x$ can be computed as the following:
\begin{align}
\PROB_{\set{X}} [v_i(\set{S}_i, \set{S}_{-i}; \mathfrak{A}, T) \leq x]
&= \PROB_{\set{X}} [\sum_{t=1}^T \gamma_t r_{I_t(\set{S}_i, \set{S}_{-i}), t} \leq x]
\\
&= \PROB_{\set{X}}\big[ \sum_{t=1}^{T}\sum_{a \in \set{S}_i}  \gamma_t  r(a) \mathbbm{1}[I_t( \set{S}_i, \set{S}_{-i}; \mathfrak{A}, T) = a ] \leq x\big]
\\
&= \PROB_{\set{X}}\big[  \big(\sum_{t=2}^{T} \gamma_t  \sum_{a \in \set{S}_i}r(a) (N^{\set{X}}_{a,t} -N^{\set{X}}_{a,t-1})\big)
+ \gamma_1  \sum_{a \in \set{S}_i}\mu(a)  N^{\set{X}}_{a,1} \leq x\big]
\label{thm:dominant_general:ineq1}
\\
&= \PROB_{\set{X}}\big[ \sum_{t=1}^{T} \sum_{a \in \set{S}_i}r(a) (\gamma_{t+1} - \gamma_t)N^{\set{X}}_{a,t} \leq x\big] \label{thm:dominant_general:ineq2}
\\
&\leq \PROB_{\set{X}'}\big[ \sum_{t=1}^{T} \sum_{a \in \set{S}'_i}r(a) (\gamma_{t+1} - \gamma_t)N^{\set{X}'}_{a,t} \leq x\big] = \PROB_{\set{X}'} [v_i(\set{S}'_i, \set{S}_{-i}; \mathfrak{A}, T) \leq x]\;,
\label{thm:dominant_general:ineq3}
\end{align}
where equation~\eqref{thm:dominant_general:ineq1} holds since $N^{\set{X}}_{a,t} = N^{\set{X}}_{a,t-1} + \mathbbm{1}[I_t=a]$ for $t\geq2$ and $N^{\set{X}}_{a,1} = \mathbbm{1}[I_1=a]$. In equation~\eqref{thm:dominant_general:ineq2} we assume that $\gamma_0 = 0$, and inequality~\eqref{thm:dominant_general:ineq3} follows from the stochastic dominance of $N^{\set{X}'}_{\set{S}'_i,t}$ over $N^{\set{X}}_{\set{S}_i,t}$ and condition 1 in theorem statement applied with~\ref{thm:dominance_preserve}.
Hence we conclude that revenue under $\set{S}'_i$ has stochastic dominance over that under $\set{S}_i$, and by Proposition~\ref{thm:dominance_to_utility}, $u_i(\set{S}'_i, \set{S}_{-i};\mathfrak{A}, T) \geq u_i(\set{S}_i, \set{S}_{-i};\mathfrak{A}, T)$ for any non-decreasing utility $U$.
\end{proof}

{\bf Remark.} It is straightforward to check that $u_i(\set{S}'_i, \set{S}_{-i};\mathfrak{A}, T) > u_i(\set{S}_i, \set{S}_{-i};\mathfrak{A}, T)$ if any of the described inequalities is strict.
For example, if there exists some $x$ and $t$ such that $\PROB [N^{\set{X}}_{a,t} \leq x] < \PROB [N^{\set{X}'}_{a,t} \leq x]$, then the dominance is strict.




\subsection{Embedding}
We now lay out the foundations in defining the embedding from a MAB process to the other MAB process which helps us to compare the agent's expected utility between two strategies.
This will mainly be used at analyzing the replication-proneness.
Given any two MAB processes $\set{X} = (\langle \cdot, \set{S}\rangle, \mathfrak{A})$ and $\set{X}' = (\langle \cdot, \set{S}'\rangle, \mathfrak{A}')$ with two horizon-free algorithms $\mathfrak{A}$ and $\mathfrak{A}'$, we define an arm-translation $\tau:\set{S} \mapsto \set{S}'$ be a mapping function between arms in MAB processes, and let $Im_{\tau}(\set{A})$ be image of $\set{A}$ under $\tau$ for any $\set{A} \subset \set{S}$. Given any sample path $r$, let $L(r)$ be its length, e.g. if $r = (i_1, x_1) \times (i_2, x_2)$ then $L(r) =2$, and $\mathfrak{T}_{\tau}(r)$ be the sequence that replaces each selected arm $a$ at round $t$ in $r$ into $\tau(a)$ for any $t \in [L(r)]$.
Now we define an {\em embedding} of MAB process $\set{X}$ into $\set{X}'$.

\begin{definition}[Embedding]\label{def:coupling}
Given an arm-translation $\tau$, we say that $\set{X}$ is embedded into $\set{X}'$ by $\tau$ if a sample path $r^{\set{X}}_t$ is realized under $\set{X}$, then sample path $r^{\set{X}'}_{t'}$ is realized under $\set{X}'$ such that its $\set{S}_{I}$-sample path is $\mathfrak{T}_{\tau}(r^{\set{X}}_t)$, i.e. $\mathfrak{T}_{\tau}(r^{\set{X}}_{t}) = c_{\set{S}_{I}}(r^{\set{X}'}_{t'})$ for $\set{S}_{I} = Im_{\tau}(\set{S})$.
\end{definition}
Note that given $\tau$, we can construct an embedding of $\set{X}$ into $\set{X}'$ by coupling the sample paths in $\set{X}$ and $\set{X}'$ as the following: at each round in $\set{X}'$, (i) arm $a \in \set{S}' \setminus \set{S}_I$ can be selected in $\set{X}'$, (ii) $a \in \set{S}_I$ can be selected in $\set{X}'$ and one of the arms in $\tau^{-1}(a)$ can be selected in $\set{X}$.
In this scenario, $\set{X}$ move forwards only when $a \in \set{S}_I$ is selected under $\set{X}'$.
Likewise, we can enforce the described mapping between any realization of sample paths regardless of MAB processes $\set{X}$ and $\set{X}'$, and hence we can always embed $\set{X}$ into $\set{X}'$. However, it is not true that such mapped pair of sample paths has the same marginal probability to be realized in corresponding MAB process, hence we introduce the following notion of proper embedding:
\begin{definition}[Proper embedding]\label{def:prop_embedding}
We say that $\set{X}$ is properly embedded into $\set{X}'$ if for any $t\geq 1$, there exists a mapping $\tau:\set{S}\mapsto \set{S}'$ such that given any sample path $r^{\set{X}}_t$ in $\set{X}$, the marginal probability that sample path $r$ is realized under $\set{X}$ whose $\mathfrak{T}_{\tau}$ transformation is $\mathfrak{T}_{\tau}(r^{\set{X}}_t)$, is equal to the summation of marginal probability that sample path $r'$ is realized of which its $\set{S}_{I}$-sample path is $\mathfrak{T}_{\tau}(r^{\set{X}}_t)$, i.e.
\begin{align*}
\PROB_{\set{X}}[\{r:\mathfrak{T}_{\tau}(r) = \mathfrak{T}_{\tau}(r^{\set{X}}_t)\}] = \PROB_{\set{X}'}[\{r':c_{\set{S}_{I}}(r') = \mathfrak{T}_{\tau}(r^{\set{X}}_t)\}]\; .
\end{align*}
\end{definition}

Unlike from embedding, proper embedding does not always exist, but rather depends on the parameters of two MAB process.

\bigskip
{\bf Illustrative examples.}
To give an intuition on when the proper embedding exists and how it can be constructed, we provide some concrete examples.

{\bf Example 1.}
Suppose that we have two MAB process $\set{X} = (\langle \cdot, \{1,2\}\rangle, \mathfrak{A})$ and $\set{X}' = (\langle \cdot, \{3,4\}\rangle, \mathfrak{A}')$ where $\mathfrak{A}$ selects the arms in uniform random manner, but $\mathfrak{A}'$ selects arm $3$ with probability $1/3$ and $4$ with probability $2/3$.
If $Im_{\tau}(\{1,2\}) = \{3\}$ or $\{4\}$, then it is obvious that it is not proper embedding since both arms need to be selected in $\set{X}'$.
Suppose that $\tau(1) = 3$ and $\tau(2) = 4$ without loss of generality and sample path $r = (1,\cdot)\times (2,\cdot)$ is given. In this case, the only sample path in $\set{X}'$ whose $Im_{\tau}(\{1,2\})$-sample path is $(1,\cdot)\times (2,\cdot)$, is $r' = (3,\cdot)\times (4,\cdot)$.
However, $\PROB_{\set{X}}[r] = 1/2\cdot 1/2 \neq 1/3\cdot2/3 = \PROB_{\set{X}'}[r']$ and hence there can be no proper embedding.

{\bf Example 2.}
Suppose that $\set{X} = (\langle \cdot, \{1,2\}, \mathfrak{A})$ and $\set{X}' = (\langle \cdot, \{1,2,3\}, \mathfrak{A}')$ where both $\mathfrak{A}$ and $\mathfrak{A}'$ chooses the arms in uniformly random manner with respect to corresponding MAB instance.
Consider $\tau:\{1,2\} \mapsto \{1,2,3\}$ such that $\tau(1) =1$ and $\tau(2) = 2$.
In this case, given any realization $r$ under $\set{X}$, we can easily check that the summation of marginal probability that sample path $r'$ under $\set{X}'$ whose $\{1,2\}$-sample path is $r$ is equal to the marginal probability of having $r$ under $\set{X}$.

{\bf Example 3.}
Suppose that $\set{X} = (\langle \cdot, \{1,2,3 \}, \mathfrak{A})$ and $\set{X}' = (\langle \cdot, \{1,2\}, \mathfrak{A}')$ where $\mathfrak{A}$ chooses arm $1$ and $2$ with probability $1/4$ and $3$ with probability $1/2$, and $\mathfrak{A}'$ chooses the arm in uniformly random manner.
Consider $\tau:\{1,2,3\} \mapsto \{1,2\}$ such that $\tau(1) = \tau(2) = 1$ and $\tau(3) = 2$. Since the summation of probability that arm $1$ or $2$ is selected in $\set{X}$ is equal to the probability that arm $1$ is selected in $\set{X}'$, it is straightforward that $\tau$ makes a proper embedding from $\set{X}$ to $\set{X}'$.
\bigskip

Proper embedding enables us to compare the probabilistic properties of random variables defined on two MAB process respectively.
Representatively, we present the following claim which holds upon surjective arm-translation function: 
\begin{claim}\label{thm:coupling_equal}
Given a MAB process $\set{X} = (\langle \cdot, \set{S}\rangle, \mathfrak{A})$ and its proper embedding $\set{X}' = (\langle \cdot, \set{S}'\rangle, \mathfrak{A}')$ with $\tau:\set{S}\mapsto \set{S}'$,
if $\tau$ is surjective and each mapped pair $a$ and $\tau(a)$ have the same reward distribution for any $a \in \set{S}$, then the principal's regret until any round $t$ has the same distribution between $\set{X}$ and $\set{X}'$ .
\end{claim}
We skip the proof since it is obvious from the definition of proper embedding.

If $\tau$ is not surjective, then there will be unmapped arms in $\set{X}'$, and this remaining arms will increase the revenue of the agent who owns these arms since they will be selected more times due to the existence of the remaining arms. This means that the agent's $t$-count can be stochastically ordered between two strategies, and we formally present this result as follows:
\begin{claim}\label{thm:embedding_to_ordering}
Given a MAB process $\set{X} = (\langle \cdot, \set{S}\rangle, \mathfrak{A})$ and its proper embedding $\set{X}' = (\langle \cdot, \set{S}'\rangle, \mathfrak{A}')$ with $\tau:\set{S}\mapsto \set{S}'$,
for any set of arms $\set{A} \subset \set{S}$, $N^{\set{X}'}_{\set{A}',t}$ has stochastic dominance over $N^{\set{X}}_{\set{A},t}$ for any $t \geq 1$ where $\set{A}' = Im_{\tau}(\set{A}) \cup (\set{S}'\setminus Im_{\tau}(\set{S}))$.
\end{claim}

\begin{proof}
Given a sample path $r^{\set{X}}_t$ in $\set{X}$, suppose that a sample path $r'$ in $\set{X}'$ satisfies that $c_{\set{A}}(r') = Im_{\tau}(r^{\set{X}}_t)$.
Let $\set{A}_I = Im_{\tau}(\set{A})$ and $\set{S}_I = Im_{\tau}(\set{S})$.
Under the embedding, it is obvious that $L(r') \geq t$ and $N^{\set{X}}_{\set{A},t} (r^{\set{X}}_t) = N^{\set{X}'}_{\set{A}_I,L(r')} (r')$. Suppose that arms in $\set{A}_I$ and $\set{S}_I \setminus \set{A}_I$ appears $x$ and $y$ times respectively from round $t+1$ to  $L(r')$ in $r'$ under $\set{X}'$. Then the following inequalities hold:
\begin{align*}
    N^{\set{X}'}_{\set{A}_I,L(r')} (r') 
    = N^{\set{X}'}_{\set{A}_I,t}(r') + x
    \leq N^{\set{X}'}_{\set{A}_I,t}(r') + x+y
    = N^{\set{X}'}_{\set{A}',t}(r')\; ,
\end{align*}
where the last equation holds since arms in $\set{S}'\setminus \set{S}_I$ appears exactly $x+y$ times from round $1$ to round $t$ in $r$ under $\set{X}'$.
Hence, the following inequalities hold:
\begin{align*}
    \PROB_{\set{X}}[N^{\set{X}}_{\set{A},t} \leq m] 
    = \sum_{j=0}^{m}\PROB_{\set{X}}[N^{\set{X}}_{\set{A},t} = j]
    &= \sum_{j=0}^{m}\sum_{N^{\set{X}}_{\set{A},t}(r^{\set{X}}_t) = j}\PROB_{\set{X}}[r^{\set{X}}_t] \\
    &=\sum_{j=0}^{m}\sum_{N^{\set{X}'}_{\set{A}_I,L(r')}(r') = j}\PROB_{\set{X}'}[r']\\
    &=\sum_{r'} \PROB_{\set{X}'}[N^{\set{X}'}_{\set{A}_I,L(r')}(r') \leq m] \\
    &\geq\sum_{r'} \PROB_{\set{X}'}[N^{\set{X}'}_{\set{A}',t}(r') \leq m] 
    = \PROB_{\set{X}'}[N^{\set{X}'}_{\set{A}',t} \leq m]\; ,
\end{align*}
and we conclude that $N^{\set{X}'}_{\set{A}',t}$ has stochastic dominance over $N^{\set{X}}_{\set{A},t}$.
\end{proof}

\section{Proof of Main Results}
\label{sec:proofs}
We now provide proofs for all the theorems presented in our main paper.

\subsection{Proof of Theorem 1}

\begin{proof}
Instead of proving Theorem 1, we prove the following generalized statement:
\begin{proposition}[Restatement of theorem 1 in main paper]
UCB1 is replication-prone under any strictly increasing utility function $U$, and $\forall \gamma \in \Gamma_T$.
\end{proposition}

We first consider a simple scenario when all the agents have single original arm.
It suffices to show that replicating $k$ times is a strictly dominant strategy over replicating $k-1$ times for $k=1,2,\ldots$.
Noting that the proof can easily be extended into the case where $|\set{S}_i| \geq 2$, we only provide the proof for the case when $|\set{S}_i| = 1$.
Let $\set{S}_i = \{1\}$ and $\set{S}_{-i} = \{2,3,\ldots, k\}$ be the union of other agents' arms.
Let $\set{S}'_i = \{1,1'\}$ be agent $i$'s another strategy where $1'$ is a replication of arm $1$. Now we want to prove that agent $i$'s payoff is strictly dominant in $\set{S}'_i$ over $\set{S}_i$. Let $\set{S} = \set{S}_i \cup \set{S}_{-i}$ and $\set{S}' = \set{S}'_i \cup \set{S}_{-i}$.
We denote MAB process for the non-replicated case be $\set{X} = (\langle \cdot, \set{S}, \set{D}\rangle, \mathfrak{A})$ and for the replicated case be $\set{X}' = (\langle \cdot, \set{S}', \set{D}'\rangle, \mathfrak{A}')$ for corresponding UCB1 algorithms $\mathfrak{A}$ and $\mathfrak{A}'$.
We now show that there exists a proper embedding from $\set{X}$ into $\set{X}'$, and conclude the proof by calculating the expected payoff of agent $i$ with the notion of $t$-count.

Firstly we prove that there exists a proper embedding from $\set{X}$ into $\set{X}'$.
Consider an arm-translation $\tau:\set{S} \mapsto \set{S}'$ such that $\tau(a) = a$ for any $a \in \set{S}$.
To show that there exists a proper embedding, since $\{r_x: \mathfrak{T}_{\tau}(r_x) = \mathfrak{T}_{\tau}(r)\} = \{r\}$ by our construction of $\tau$,
we need to show that for any valid sample path $r$ in $\set{X}$,
\begin{align}
    \PROB_{\set{X}}[r] =  
    \sum_{c_{\set{S}}(r') = \mathfrak{T}_{\tau}(r)}\PROB_{\set{X}'}[r']\; .\label{thm:negative:proof:eq1}
\end{align}
Since $\tau$ is identity function we simply use $r$ instead of $\mathfrak{T}_{\tau}(r)$.
We use the proof by induction with respect to the length of sample path in $\set{X}$.
Assume that the sample path $r$ under $\set{X}$ is given and its length is $1$, i.e. $L(r) = 1$.
Let the arm selected in $r$ be $a \in \set{S}$.
Then,
\begin{align*}
\sum_{c_{\set{S}}(r') = (r)}\PROB_{\set{X}'}[r']
&= \PROB_{\set{X}'}[I_1 = a] + \PROB_{\set{X}'}[\{I_1 \neq a\} \cap \{I_2 = a\}] \\
&= 1/(k+1) + 1/(k+1)\cdot 1/k 
= \PROB_{\set{X}}[I_1 = a] 
= \PROB_{\set{X}}[r]\; ,
\end{align*}
and hence equation~\eqref{thm:negative:proof:eq1} holds for $L(r)=1$.
Now suppose that equation~\eqref{thm:negative:proof:eq1} holds for any sample path $r$ with $L(r) = m$.
Again, let $r$ be the sample path of length $m+1$ under $\set{X}$, and $r_m$ be its sub-sequence that only accounts for the first $m$ elements in $r$, e.g. if $r = (1,0.1)\times \ldots \times (m,0.1)\times (m+1,0.1)$ then $r_m = (1,0.1)\times \ldots \times (m,0.1)$. Let the arm selected in $r$ at round $m+1$ be $a \in \set{S}$. Then the following equations hold:
\begin{align}
\PROB_{\set{X}}[r] 
= \PROB_{\set{X}}[r | r_m]\PROB_{\set{X}}[r_m] 
&= \PROB_{\set{X}}[I_{m+1} = a | r_m]\PROB_{\set{X}}[r_m] \\
&= \sum_{c_{\set{S}}(r')=r_m}\PROB_{\set{X}'}[r'] \cdot \PROB_{\set{X}}[I_{m+1} = a | r_m]\; .\label{thm:negative:proof:eq2}
\end{align}
Now we look into $\PROB_{\set{X}}[I_{m+1} = a | r_m]$.
For $s \in \set{S}$, let arm $s$'s UCB index at round $m$ be $u_m(s)$.
If $u_m(a) < \argmax_{s \in \set{S}} (u_m(s)) $, then arm $a$ cannot be selected at round $m+1$ and hence $r$ can't be realized under $\set{X}$, which is a contradiction.
Hence we have $u_m(a) \geq \argmax_{s \in \set{S}} (u_m(s))$, and let the set of arms in this tied index be $\set{T}$.
Given $r'$ with $c_{\set{S}}(r') = r_m$, suppose that arm $1'$ has the highest UCB index from round $L(r)$ to $L(r)+t-1$, and hence selected by UCB1 at rounds $L(r)+1$ to $L(r)+t$, and not for round $L(r)+t$. It is obvious that such $t< \infty$ exists since arm $1'$ and $1$ follows the same reward distribution.
In this case at round $L(r)+t+1$, arm $1'$ cannot have the highest UCB index, and hence one of the arms in $\set{T}$ would be selected since $c_{\set{S}}(r') = r_m$ and only arm $1'$ is selected from round $L(r)+1$ to $L(r)+t$. Hence, the probability that arm $a$ is selected in this case is equal to $\PROB_{\set{X}}[I_{m+1}=a|r_m]$, and we have:
\begin{align}
\sum_{c_{\set{S}}(r')=r_m}\PROB_{\set{X}'}[r'] \cdot \PROB_{\set{X}}[I_{m+1} = a | r_m] 
&= \sum_{c_{\set{S}}(r')=r_m}\PROB_{\set{X}'}[r']/|\set{T}|\\
&= \sum_{c_{\set{S}}(r')=r_m}\PROB_{\set{X}'}[r'] \cdot \PROB_{\set{X}'}[I_{L(r)+t+1} = a |r']\\
&= \sum_{c_{\set{S}}(r')=r}\PROB_{\set{X}'}[r']\label{thm:negative:proof:eq3}\; .
\end{align}
By equation~\eqref{thm:negative:proof:eq2} and~\eqref{thm:negative:proof:eq3}, we have $\PROB_{\set{X}}[r] = \sum_{c_{\set{S}}(r')=r}\PROB_{\set{X}'}[r']$ and we conclude that equation~\eqref{thm:negative:proof:eq1} is true.
Hence, by Claim~\ref{thm:dominant_general} and~\ref{thm:embedding_to_ordering}, $\set{S}'_i$ is a dominant strategy over $\set{S}_i$ regardless of the MAB instance, and now it is enough to prove that there exists a MAB instance such that equality does not hold in $u_i(\set{S}'_i, \set{S}_{-i}) \geq u_i(\set{S}_i, \set{S}_{-i})$.
In Bernoulli bandit case, it is obvious that $\PROB_{\set{X}'}[N^{\set{X}'}_{\set{S}'_i,t} \geq 2]  = 1 > \PROB_{\set{X}}[N^{\set{X}}_{\set{S}_i,t} \geq 2]$ for any $t\geq n$.
Then in inequality~\eqref{thm:dominant_general:ineq3} in the proof of Claim~\ref{thm:dominant_general}, the equality cannot hold and hence $\set{S}'_i$ is a strict dominant strategy over $\set{S}_i$, which implies that $u_i(\set{S}'_i, \set{S}_{-i}) > u_i(\set{S}_i, \set{S}_{-i})$ for any time horizon $T$ by Proposition~\ref{thm:dominance_to_utility}.

Finally, we can extend our proof into the multiple arm cases by constructing $\set{S}'$ which replicates all the arms in $\set{S}_i$ once. We omit the detailed process since it can easily be derived.
\end{proof}

\subsection{Proof of Theorem 2}
\begin{proof}
We prove the following generalized statement:
\begin{proposition}[Restatement of theorem 2 in main paper]
H-UCB is replication-proof  for $i \in \set{N}$ under any non-decreasing utility function $U$, $\forall T \geq n$ and $\forall \gamma \in \Gamma_T$.
\end{proposition}

Let $\set{X} = (\langle \cdot, \set{S}_i\cup \set{S}_{-i}\rangle, \mathfrak{A})$ and $\set{X}' = (\langle \cdot, \set{S}'_i\cup \set{S}_{-i}\rangle, \mathfrak{A})$ where $\mathfrak{A}$ is H-UCB.
By Claim~\ref{thm:dominant_general}, it is enough to show that $N^{\set{X}}_{\set{S}_i, t}$ has stochastic dominance over $N^{\set{X}'}_{\set{S}'_i, t}$ for any strategy $\set{S}'_i$.
To this end, we show that $\mathbbm{1}[I^{\set{X}}_t(\set{S}_i, \set{S}_{-i}) \in \set{S}_i] \succeq \mathbbm{1}[I^{\set{X}'}_t(\set{S}'_i, \set{S}_{-i}) \in \set{S}'_i]$ for any round $t$.
We use proof by induction.
For $t \in [n]$, it is straightforward that $\mathbbm{1}[I^{\set{X}}_t(\set{S}_i, \set{S}_{-i}) \in \set{S}_i] \sim  \mathbbm{1}[I^{\set{X}'}_t(\set{S}'_i, \set{S}_{-i}) \in \set{S}'_i]$ due to the agent initialization step in H-UCB.
Suppose that we have $\mathbbm{1}[I^{\set{X}}_t(\set{S}_i, \set{S}_{-i}) \in \set{S}_i] \succeq \mathbbm{1}[I^{\set{X}'}_t(\set{S}'_i, \set{S}_{-i}) \in \set{S}'_i]$ holds for any $t \in [t']$.
Then, by Proposition~\ref{thm:dominance_preserve}, we can derive the following relation between $t$-count of $\set{S}_i$ and $\set{S}'_i$ for any $t \in [t']$:
\begin{align*}
N^{\set{X}}_{\set{S}_i, t}  = \sum_{q=1}^t \mathbbm{1}[I^{\set{X}}_{q}(\set{S}_i, \set{S}_{-i}) \in \set{S}_i] \succeq  \sum_{q=1}^t \mathbbm{1}[I^{\set{X}'}_{q}(\set{S}'_i, \set{S}_{-i}) \in \set{S}'_i] = N^{\set{X}'}_{\set{S}'_i, t}\; .
\end{align*}

The probability that agent $i$ is selected at round $t'+1$ can be computed as the following\footnote{In precise, we need to integrate over $r\geq 0$ rather than summing over it since the support of $\bar{r}^{\set{X}}_{i,m}$ is continuous over non-negative real value, however, we abuse it for the sake of simplicity.}:
\begin{align}
\hspace{-1cm}\PROB[I^{\set{X}}_{t'+1}(\set{S}_i, \set{S}_{-i}) \in \set{S}_i]
&= \sum_{m=0}^{t'} \PROB \big[ I^{\set{X}}_{t'+1}(\set{S}_i, \set{S}_{-i}) \in \set{S}_i | N^{\set{X}}_{\set{S}_i,t'} = m\big] \PROB[N^{\set{X}}_{\set{S}_i,t'} = m]\nonumber
\\
&= \sum_{m=0}^{t'} \sum_{r\geq 0} \PROB\big[ I^{\set{X}}_{t'+1}(\set{S}_i, \set{S}_{-i}) \in \set{S}_i | N^{\set{X}}_{\set{S}_i,t'} = m, \bar{r}^{\set{X}}_{i,m} = r \big]\PROB[\bar{r}^{\set{X}}_{i,m} = r   | N^{\set{X}}_{\set{S}_i,t'} = m]\PROB[N^{\set{X}}_{\set{S}_i,t'} = m]\; ,\label{thm:dominance_tcount:ineq1}
\end{align}
where $\bar{r}^{\set{X}}_{i,m}$ denotes empirical average reward of agent $i$ given that agent $i$ is selected $m$th time by H-UCB under MAB process $\set{X}$.
Since $X_{i,1} \succeq X_{i,k}$ for any $k \in [l(i)]$, it is straightforward that $\bar{r}^{\set{X}}_{i,m}$ conditioned to $N^{\set{X}}_{i,t'}=m$ has stochastic dominance over $\bar{r}^{\set{X}'}_{i,m}$ conditioned to $N^{\set{X}'}_{i,t'} = m$, i.e.
\begin{align}
\PROB[\bar{r}^{\set{X}}_{i,m} > r   | N^{\set{X}}_{\set{S}_i,t'} = m]
\geq
\PROB[\bar{r}^{\set{X}'}_{i,m} > r   | N^{\set{X}'}_{\set{S}'_i,t'} = m]\; ,\label{thm:dominance_tcount:ineq2}
\end{align}
for any $r \geq 0$.

Besides, due to the nature of phase 1 that it selects agent solely based on agent's $t$-count and empirical average reward, the following holds for any $r_1 \geq r_2$:
\begin{align}
\PROB\big[ I^{\set{X}}_{t'+1}(\set{S}_i, \set{S}_{-i}) \in \set{S}_i | N^{\set{X}}_{\set{S}_i,t'} = m, \bar{r}_{i,m} = r_1 \big] &\geq \PROB\big[ I^{\set{X}}_{t'+1}(\set{S}_i, \set{S}_{-i}) \in \set{S}_i | N^{\set{X}}_{\set{S}_i,t'} = m, \bar{r}^{\set{X}}_{i,m} = r_2 \big]\; .\label{thm:dominance_tcount:eq1}
\end{align}

Now we observe that the distribution of sample path for the arms in $\set{S}'_i$ conditioned on agent $i$'s $t$-count and empirical average reward would be the same\footnote{One may consider a sort of partial proper embedding from $\set{X}$ to $\set{X}'$ so that any arm $a \in \set{S}'_i$ in $\set{X}$ is mapped to
$a \in \set{S}'_i$ in $\set{X}'$, which implies that the selection and realization of the arms in $\set{S}'_i$ can be coupled between $\set{X}$ and $\set{X}'$.} between $\set{X}$ and $\set{X}'$. 
This implies that the probability that agent $i$ is selected at round $t$ given its $t$-count and empirical average reward is the same between $\set{X}$ and $\set{X}'$, and hence equation~\eqref{thm:dominance_tcount:eq1} can be lower-bounded as the following for any $r_1 \geq r_2$:
\begin{align}
\PROB\big[ I^{\set{X}}_{t'+1}(\set{S}_i, \set{S}_{-i}) \in \set{S}_i | N^{\set{X}}_{\set{S}_i,t'} = m, \bar{r}_{i,m} = r_1 \big] &\geq \PROB\big[ I^{\set{X}'}_{t'+1}(\set{S}'_i, \set{S}_{-i}) \in \set{S}'_i | N^{\set{X}'}_{\set{S}'_i,t'} = m, \bar{r}^{\set{X}'}_{i,m} = r_2 \big]\; . \label{thm:dominance_tcount:ineq3}
\end{align}

Since we have $N^{\set{X}}_{\set{S}_i,t} \succeq N^{\set{X}'}_{\set{S}'_i,t'}$ for any $t \in [t']$ by our induction hypothesis, plugging inequalities~\eqref{thm:dominance_tcount:ineq2} and~\eqref{thm:dominance_tcount:ineq3} into~\eqref{thm:dominance_tcount:ineq1} yields the following:
\begin{align*}
\hspace{-1cm}\PROB[I^{\set{X}}_{t'+1}(\set{S}_i, \set{S}_{-i}) \in \set{S}_i]
&= \sum_{m=0}^{t} \sum_{r\geq 0} \PROB\big[ I^{\set{X}}_{t'+1}(\set{S}_i, \set{S}_{-i}) \in \set{S}_i | N^{\set{X}}_{\set{S}_i,t'} = m, \bar{r}^{\set{X}}_{i,m} = r \big]\PROB[\bar{r}^{\set{X}}_{i,m} = r   | N^{\set{X}}_{\set{S}_i,t'} = m]\PROB[N^{\set{X}}_{\set{S}_i,t} = m]\\
&\geq \sum_{m=0}^{t} \sum_{r} \PROB\big[ I^{\set{X}'}_{t'+1}(\set{S}'_i, \set{S}_{-i}) \in \set{S}'_i | N^{\set{X}'}_{i,t'} = m, \bar{r}_{i,m} = r \big]\PROB[\bar{r}_{i,m} = r   | N^{\set{X}'}_{i,t'} = m]\PROB[N^{\set{X}'}_{\set{S}'_i,t'} = m]
\\
&= \PROB[I^{\set{X}'}_{t+1}(\set{S}'_i, \set{S}_{-i}) \in \set{S}'_i]\; .
\end{align*}

Hence we conclude that $\mathbbm{1}[I^{\set{X}}_t (\set{S}_i, \set{S}_{-i}) \in \set{S}_i] \succeq \mathbbm{1}[I^{\set{X}'}_t (\set{S}'_i, \set{S}_{-i}) \in \set{S}'_i]$ for any $t \geq 1$, which implies that $N^{\set{X}}_{\set{S}_i, t} \succeq N^{\set{X}'}_{\set{S}'_i, t}$ for any $t \geq 1$ by Proposition~\ref{thm:dominance_preserve}. 
By Claim~\ref{thm:dominant_general}, we conclude that H-UCB is replication-proof.
\end{proof}

\subsection{Proof of Theorem 3}
Firstly, under the equilibrium $\set{S} = \cup_{i \in \set{N}} \set{S}_i$ such that $\set{S}_i =\{o_{i,1}\}$ for any $i \in \set{N}$, it is straightforward to check that H-UCB submits the proposed regret bound since it exactly reduces to UCB1 under the canonical setting.
Now we characterize all the possible dominant strategy equilibrium, and then show that there exists a proper embedding for any pair of equilibrium. 
Suppose that $\set{S}'_i$ possesses a suboptimal arm $a$. Then, inequality~\eqref{thm:dominance_tcount:ineq2} in the proof of Theorem 2 will obviously become strict. Following the similar analysis, we can observe that the selection probability of agent $i$ under $\set{S}'_i$ will be strictly smaller than that under $\set{S}_i$, which eventually implies that $\set{S}'_i$ cannot be a dominant strategy.
Finally, we prove that the principal's regret remains the same as $\set{S}_i$ under any strategy $\set{S}''_i = \{o^{(1)}_{i,1}, o^{(2)}_{i,1}, \ldots, o^{(c_{i,1})}_{i,1}\}$. Let $\set{X}'' = (\langle \cdot, \set{S}''_i, \set{S}_{-i}\rangle, \mathfrak{A})$, and $\set{S}'' = \set{S}''_i \cup \set{S}_{-i}$.
Consider arm-translation $\tau:\set{S}'' \rightarrow  \set{S}$ which maps all the arms in $\set{S}''_i$ to $o_{i,1}$ in $\set{S}_i$ and identically maps the arms in $\set{S}_{-i}$. Then, it is straightforward to check that this constructs a proper embedding from $\set{X}''$ to $\set{X}$. Hence by Claim~\ref{thm:coupling_equal}, the principal's regret remains the same under $\set{S}$ and $\set{S}''$, and we conclude the proof.





\subsection{Proof of Theorem 5}
\begin{proof}
For simplicity, we assume that $\set{S}_i$ is ordered so that $\mu(s_{i,1}) \geq \mu(s_{i,2}) \geq \ldots \geq \mu(s_{i,l(i)})$ for $i \in \set{N}$.
For simplicity, we abuse $N_{\set{A},t}$ to denote $N^{\set{X}}_{\set{A},t}$.
We use$N_{i,t}$, $N_{i_a,t}$ to denote the expected number of times agent $i$ and arm $i_a$  are played up to time $t$, and $R_{i,t}$, $R_{i_a,t}$ be the empirical average reward of agent $i$ and arm $i_a$ at round $t$, respectively. Hence, if the index is contained in $\set{N}$, it refers to the agent's random variable, otherwise, it indicates to that of the arm.
Let $i^\star$ be the agent with optimal arm, and $a^\star$ be the optimal arm.
Let $\Delta_{i} = \mu(s_{i^\star, 1}) -  \mu(s_{i,1})$, $\Delta_{i_a} = 
\mu(s_{i^\star, 1})  - \mu(i_a)$,  $\Delta^m_{i} = \mu(s_{i^\star, 1}) - \mu(s_{i, l(i)})$, and $\delta_{i_a} = \mu(s_{i,1})- \mu(i_a)$.

Thanks to our analysis given in so far, any strategic agent can be viewed as a replicator with only single original arm, since it will only register the best arm or any set of replicas of it.Hence we assume that there only exists replicators without loss of generality.\footnote{Though, we can tighten our analysis by separating the regret from each type since the regret bound for strategic agent would definitely be smaller.}

Let $Y_i$ be the event that $\set{B}_i$ contains $i$'s optimal arm or any replica of it.
We begin with the following two lemmas which gives upper-bound for expected regret occurred from replicator $i$.
\begin{lemma}\label{lem:prob_noopt}
$\PROB[Y_i^c] \leq 1/T$\; .
\end{lemma}
\begin{proof}[Proof of Lemma~\ref{lem:prob_noopt}]
If $|\set{S}_i| \leq L \ln T$, then $\PROB[Y_i] = 1$.
If $\set{S}_i$'s cardinality is infinite, then $\PROB[Y_i^c] = (1-c)^{L\ln T}$.
Otherwise, the probability that $i$'s optimal arm is not sampled $\set{B}_i$ is given by
\begin{align*}
    \PROB[Y_i^c] 
    &\leq (1-c)\frac{|\set{S}_i|(1-c) -1}{|\set{S}_i|-1}\frac{|\set{S}_i|(1-c) -2}{|\set{S}_i|-2}\ldots \frac{|\set{S}_i|(1-c) -L\ln T}{|\set{S}_i|-L\ln T}
    \\
    &\leq (1-c)^{L\ln T}.
\end{align*}
Since $(1-1/x)^x \leq 1/e$ for any $x \geq 0$, we can derive the following bound on $(1-c)^{L\ln T}$:
\begin{align*}
(1-c)^{L\ln T} &\leq (1-c)^{1/c\ln T} \leq (\frac{1}{e})^{\ln T} = 1/T\; .
\end{align*}
\end{proof}

\begin{lemma}\label{lem:conditional_replicator}
Given that event $Y_i$ occurs, its conditional expected internal regret is bounded as the following:
\begin{align*}
\sum_{i_a \in \set{B}_i}\EXP [\delta_{i_a} N_{i_a,t} |  Y_i] \leq L\ln T\Big(\sum_{\substack{i_a\in \set{O}_i \\ \mu_{i_a} < \mu_{i_1}}}  \frac{8\ln N_{i,t}}{\delta_{i_a}}  + (1+\frac{\pi^2}{3})\delta_{i_a}\Big)\; .
\end{align*}
\end{lemma}
\begin{proof}[Proof of Lemma~\ref{lem:conditional_replicator}]
If the optimal arm is sampled at least one time, then all the other sampled arms induce $O(\ln T)$ internal regret~\cite{mab:ucb1} since the behavior of phase 2 in RH-UCB given the sampled arms is equal to that of UCB1 under the canonical MAB problem.
Hence, we have
\begin{align*}
\sum_{i_a \in \set{B}_i}\EXP[\delta_{i_a} N_{i_a,t} | Y_i] &\leq  L\ln T\Big(\sum_{\substack{i_a\in \set{O}_i \\ \mu_{i_a} < \mu_{i_1}}} \frac{8 \ln N_{i,t}}{\delta_{i_a}} + (1+\frac{\pi^2}{3})\delta_{i_a}\Big)\; ,
\end{align*}
from the regret analysis of~\cite{mab:ucb1}
\end{proof}

Now, we prove the following bounds on the empirical average reward of any agent.
\begin{lemma}\label{thm:H-UCBgeneral:cl1}
Given any agent $i \in \set{N}$, for each arm $j$ in $j \in \set{S}_i$ at time $t$, we have
\begin{align*}
    \hspace{-2cm}\PROB\Big[| R_{i,t} - \mu_{i_1}| \geq \sqrt{\frac{\sqrt{t}\ln t}{N_{i,t}}}\Big] \leq 1/T + 2L\ln Te^{-\sqrt{t}\ln t /(2L^2\ln^2T)} + L\ln T\Big(\sum_{\substack{i_a\in \set{O}_i \\ \mu_{i_a} < \mu_{i_1}}}  \frac{16\ln N_{i,t}}{\delta_{i_a}}  + (2+\frac{2\pi^2}{3})\delta_{i_a}\Big) / \sqrt{N_{i,t}\sqrt{t}\ln t}\; .
\end{align*}
\end{lemma}
\begin{proof}[Proof of Lemma~\ref{thm:H-UCBgeneral:cl1}]
We can divide the probability in our lemma as the following:
\begin{align}
\PROB\Big[| R_{i,t} - \mu_{i_1}| \geq \sqrt{\frac{\sqrt{t}\ln t}{N_{i,t}}}\Big] 
= \PROB\Big[Y_i \cap |R_{i,t} - \mu_{i_1}| \geq \sqrt{\frac{\sqrt{t}\ln t}{N_{i,t}}}\Big] 
+ \PROB\Big[Y_i^c \cap |R_{i,t} - \mu_{i_1}| \geq \sqrt{\frac{\sqrt{t}\ln t}{N_{i,t}}}\Big]\; .\label{thm:ucb2phaesgeneral:cl1:ineq0}
\end{align}

The latter term is at most $1/T$ by Lemma~\ref{lem:prob_noopt}.
For the former one, we begin with the following inequalities where we use $\PROB_{Y_i}[X] = \PROB[X | Y_i]$ for simplicity:
\begin{align}
&\PROB\Big[Y_i \cap | R_{i,t} - \mu_{i_1}| \geq \sqrt{\frac{\sqrt{t}\ln t}{N_{i,t}}}\Big] \nonumber \\
&\leq
\PROB_{Y_i}\Big[| N_{i,t}R_{i,t} - N_{i,t}\mu_{i_1}| \geq \sqrt{N_{i,t}\sqrt{t}\ln t} \Big]
\nonumber\\
&= 
\PROB_{Y_i}\Big[| \sum_{a=1}^{\set{S}_i}N_{i_a, t}R_{i_a, t} - N_{i,t}\mu_{i_1}| \geq \sqrt{N_{i,t}\sqrt{t}\ln t} | Y_i\Big]
\nonumber\\
&=
\PROB_{Y_i}\Big[| \sum_{i_a \in \set{B}_i}N_{i_a, t}R_{i_a, t}  -\sum_{i_a \in \set{B}_i}N_{i_a, t}\mu_{i_a} + 
 \sum_{i_a \in \set{B}_i}N_{i_a, t}\mu_{i_a} - \mu_{i_1}\sum_{i_a \in \set{B}_i}N_{i_a,t}| \geq \sqrt{N_{i,t}\sqrt{t}\ln t}\Big]
\nonumber\\
&\leq
\PROB_{Y_i}[| \sum_{i_a \in \set{B}_i}N_{i_a,t}(R_{i_a,t} - \mu_{i_a})|\geq \sqrt{\frac{N_{i,t}\sqrt{t}\ln t}{4}}]
+ 
\PROB_{Y_i}\Big[| \sum_{i_a \in \set{B}_i}N_{i_a,t}(\mu_{i_1} - \mu_{i_a}) |  \geq \sqrt{\frac{N_{i,t}\sqrt{t}\ln t}{4}}\Big]
\label{thm:ucb2phaesgeneral:cl1:ineq1.5}\\
&\leq
\PROB_{Y_i}\Big[\sum_{i_a \in \set{B}_i}| N_{i_a,t}(R_{i_a,t} - \mu_{i_a})|\geq \sqrt{\frac{N_{i,t}\sqrt{t}\ln t}{4}}\Big]
+ 
\PROB_{Y_i}\Big[\sum_{i_a \in \set{B}_i}N_{i_a,t}(\mu_{i_1} - \mu_{i_a})  \geq \sqrt{\frac{N_{i,t}\sqrt{t}\ln t}{4}}\Big]
\label{thm:ucb2phaesgeneral:cl1:ineq1}\\
&\leq
\sum_{i_a \in \set{B}_i}
\PROB_{Y_i}\Big[|(R_{i_a,t} - \mu_{i_a})| \frac{N_{i_a,t}}{N_{i,t}}\geq 
\sqrt{\frac{\sqrt{t}\ln t}{4|\set{B}_i^2|N_{i,t}}}\Big]
+ 
\PROB_{Y_i}\Big[\sum_{\substack{i_a\in \set{B}_i \\ \mu_{i_a} < \mu_{i_1}}}\delta_{i_a}N_{i_a,t}\geq 
\sqrt{\frac{N_{i,t}\sqrt{t}\ln t}{4}}\Big]\; ,\label{thm:ucb2phaesgeneral:cl1:ineq2}
\end{align}
where we use $\PROB(\sum_i^n X_i \geq x) \leq \sum_i^n\PROB(X_i \geq x/n)$ for inequality~\eqref{thm:ucb2phaesgeneral:cl1:ineq1.5} and~\eqref{thm:ucb2phaesgeneral:cl1:ineq2}, and inequality~\eqref{thm:ucb2phaesgeneral:cl1:ineq1} holds by triangle inequality.
The first term can be bounded by Chernoff-Hoeffding inequality as follow:
\begin{align}
    \hspace{-1cm}\PROB_{Y_i}\Big[|R_{i_a, t} - \mu_{i_a}|\cdot \frac{N_{i_a, t}}{N_{i,t}} \geq \sqrt{\frac{\sqrt{t}\ln t}{4|\set{B}_i|^2N_{i,t}}}\Big]
    \leq \PROB_{Y_i}\Big[|R_{i_a, t} - \mu_{i_a}|\cdot N_{i_a, t} \geq \sqrt{\frac{N_{i,t}\sqrt{t}\ln t}{4L^2\ln^2T}}\Big]
    &\leq 2\exp\big(-\frac{\sqrt{t}\ln t N_{i,t}}{2L^2\ln^2TN_{i_a,t}}\big)\nonumber
    \\
    &\leq 2\exp\big(-\frac{\sqrt{t}\ln t }{2L^2\ln^2T}\big)\; .\label{thm:H-UCBgeneral:ineq2}
\end{align}

For the second term in inequality~\eqref{thm:ucb2phaesgeneral:cl1:ineq2}, the following holds:
\begin{align} 
\PROB_{Y_i}\Big[\sum_{\substack{a\in \set{B}_i \\ \mu_{i_a} < \mu_{i_1}}}\delta_{i_a}N_{i_a, t} \geq \sqrt{N_{i,t}\sqrt{t}\ln t/4}\Big]
&\leq 2\EXP\Big[\sum_{\substack{a\in \set{B}_i\\ \mu_{i_a} < \mu_{i_1}}}\delta_{i_a}N_{i_a, t} \Big| Y_i\Big] / \sqrt{N_{i,t}\sqrt{t}\ln t} \label{thm:H-UCBgeneral:ineq3} \\
&\leq L\ln T\Big(\sum_{\substack{i_a\in \set{O}_i \\ \mu_{i_a} < \mu_{i_1}}}  \frac{16\ln N_{i,t}}{\delta_{i_a}}  + (2+\frac{2\pi^2}{3})\delta_{i_a}\Big) / \sqrt{N_{i,t}\sqrt{t}\ln t}\; , \label{thm:H-UCBgeneral:ineq4}
\end{align}
where inequality~\eqref{thm:H-UCBgeneral:ineq3} holds by conditional Markov inequality, and~\eqref{thm:H-UCBgeneral:ineq4} follows from Lemma~\ref{lem:conditional_replicator}.
Combining inequalities~\eqref{thm:ucb2phaesgeneral:cl1:ineq0},\eqref{thm:ucb2phaesgeneral:cl1:ineq2},\eqref{thm:H-UCBgeneral:ineq2}, and~\eqref{thm:H-UCBgeneral:ineq4}, we conclude the proof.
\end{proof}

In overall, our total regret can be divided into the regret (i) that internally incurs from any optimal agent\footnote{In here, optimal agent refers to the agent with at least one optimal arm.} during identifying his optimal arm, and (ii) that externally incurs from selecting any suboptimal agents.
We can simply bound the former one as the following, which is $O(\ln^2T)$:
\begin{align}
\EXP [Regret_T^{opt} ] \leq T \cdot 1/T + L\ln T\Big(\sum_{\substack{i_a\in \set{O}_i \\ \mu_{i_a} < \mu_{i_1}}}  \frac{16\ln T}{\delta_{i_a}}  + 9\delta_{i_a}\Big)\; , \label{eq:reg_opt}
\end{align}
by the fact that the  expected regret given $Y_i^c$ is at most $T$, $\PROB[Y_i^c] \leq 1/T$, and Lemma~\ref{lem:conditional_replicator}. Hence we focus on the regret from any suboptimal agent.

Using Lemma~\ref{thm:H-UCBgeneral:cl1}, we upper bound the probability that suboptimal replicator $i$ is selected given that replicator $i$ and optimal replicator $i^\star$ has already selected at least certain times.
For notational simplicity, we denote the bounding probability (terms in RHS) in Lemma~\ref{thm:H-UCBgeneral:cl1} be $\rho_{i,t}(N_{i,t})$, i.e.
\begin{align*}
    \rho_{i,t}(N_{i,t}) = 1/T + 2L\ln Te^{-\sqrt{t}\ln t /(2L^2\ln^2T)} + L\ln T\Big(\sum_{\substack{i_a\in \set{O}_i \\ \mu_{i_a} < \mu_{i_1}}}  \frac{16\ln N_{i,t}}{\delta_{i_a}}  + (2+\frac{2\pi^2}{3})\delta_{i_a}\Big) / \sqrt{N_{i,t}\sqrt{t}\ln t}\; .
\end{align*}

We abuse $I_t$ to denote the arm selected at round $t$ if the context is clear.
We define $UCB_{i,t}$ to denote the UCB score of agent $i$ at round $t$, i.e. $UCB_{i,t} = R_{i,t} + \sqrt{\frac{\sqrt{t}\ln t}{N_{i,t}}}$..
The following lemma upper-bound the probability that any suboptimal replicator $i$ is selected given that replicator $i$ and any optimal replicator $i^\star$ is sufficiently selected:

\begin{lemma}\label{thm:H-UCBgeneral:cl2}
At round $t$, if any suboptimal replicator $i$ has been played at least $\frac{4\sqrt{t}\ln t}{\Delta_{i}^2}$ times and optimal replicator $i^\star$ has been played at least $\sqrt{t}\ln t$, then $UCB_{i,t} < UCB_{i^\star,t}$ with probability at least $1- (\rho_{i,t}(\frac{4\sqrt{t}\ln t}{\Delta_{i}^2}) + \rho_{i^\star, t}(\sqrt{t}\ln t))$, i.e.
\begin{align}
    \PROB\Big[I_{t+1} \in \set{S}_i | N_{i,t} \geq \frac{4\sqrt{t}\ln t}{\Delta_{i}^2}, N_{i^\star,t} \geq \sqrt{t}\ln t\Big] \leq \rho_{i,t}(\frac{4\sqrt{t}\ln t}{\Delta_i^2}) + \rho_{i^\star, t}(\sqrt{t}\ln t)
\end{align}
\end{lemma}
\begin{proof}[Proof of Lemma~\ref{thm:H-UCBgeneral:cl2}]
Given $ N_{i,t} \geq \frac{4\sqrt{t}\ln t}{\Delta_i^2}$ and $N_{i^\star,t} \geq \sqrt{t}\ln t$, 
assume that event $| R_{i,t} - \mu_{i_1}| \leq \sqrt{\frac{\sqrt{t}\ln t}{N_{i,t}}}$ occurs for replicator $i$ and $i^\star$. Then we have
\begin{align*}
UCB_{i,t} = R_{i,t} + \sqrt{\frac{\sqrt{t}\ln t}{N_{i,t}}} \leq R_{i,t} + \frac{\Delta_i}{2} \leq \mu_{i_1} + \sqrt{\frac{\sqrt{t}\ln t}{N_{i,t}}} + \frac{\Delta_i}{2}\leq (\mu_{i_1} + \frac{\Delta_i}{2}) + \frac{\Delta_i}{2} = \mu^\star < UCB_{i^\star, t}\; ,
\end{align*}
where the first inequality holds by $N_{i,t} \geq \frac{4\sqrt{t}\ln t}{\Delta_i^2}$, the second holds by $| R_{i,t} - \mu_{i_1}| \leq \sqrt{\frac{\sqrt{t}\ln t}{N_{i,t}}}$, the third again holds by $N_{i,t} \geq \frac{4\sqrt{t}\ln t}{\Delta_i^2}$, and the last follows from $| R_{i^\star,t} - \mu_{i^\star_1}| \leq \sqrt{\frac{\sqrt{t}\ln t}{N_{i^\star,t}}}$.

By union bound, the probability of $| R_{i,t} - \mu_{i_1}| > \sqrt{\frac{\sqrt{t}\ln t}{N_{i,t}}}$ or $| R_{i^\star,t} - \mu_{i^\star_1}| > \sqrt{\frac{\sqrt{t}\ln t}{N_{i^\star,t}}}$ is at most $\rho_{i,t}(N_{i,t}) + \rho_{i^\star,t}(N_{i^\star,t})$.
Since $\rho_{i,t}(\sqrt{t}\ln t)$ is a decreasing function on $t$ for $t\geq 1$, we conclude the proof.
\end{proof}

Now for any suboptimal agent $i$, we upper bound the expected number of times it is played up to round $T$ under H-UCB as the following:
\begin{align}
    \EXP[N_{i,T}]
    &= 1 + \EXP\Big[\sum_{t=n}^T \mathbbm{1}(I_{t+1} \in \set{S}_i)\Big]\nonumber
    \\
    &= 1 + \EXP\Big[\sum_{t=n}^T \mathbbm{1}(I_{t+1} \in \set{S}_i, N_{i,t} < \frac{4\sqrt{t}\ln t}{\Delta_i^2})\Big] + \EXP[\sum_{t=n}^T \mathbbm{1}(I_{t+1} \in \set{S}_i, N_{i,t} \geq \frac{4\sqrt{t}\ln t}{\Delta_i^2})]\nonumber\\
    &\leq \frac{4\sqrt{T}\ln T}{\Delta_i^2} + \EXP[\sum_{t=n}^T \mathbbm{1}(I_{t+1} \in \set{S}_i, N_{i,t} \geq \frac{4\sqrt{t}\ln t}{\Delta_i^2})] \nonumber\\
    &\leq \begin{aligned}[t]
      \frac{4\sqrt{T}\ln T}{\Delta_i^2} + \EXP\Big[\sum_{t=n}^T \mathbbm{1}(I_{t+1} &\in \set{S}_{i^\star}, N_{i,t} \geq \frac{4\sqrt{t}\ln t}{\Delta_i^2}, N_{i^\star,t} < \sqrt{t}\ln t) \Big]\\
      &+ \EXP\Big[\sum_{t=n}^T \mathbbm{1}(I_{t+1} \in \set{S}_i, N_{i,t} \geq \frac{4\sqrt{t}\ln t}{\Delta_i^2}, N_{i^\star,t}\geq \sqrt{t}\ln t) \Big]
      \end{aligned}\label{thm:H-UCBgeneral:ineq5}\\
    &\leq \frac{4\sqrt{T}\ln T}{\Delta_i^2} + \sqrt{T}\ln T + \EXP\Big[\sum_{t=n}^T \mathbbm{1}(I_{t+1} \in \set{S}_i, N_{i,t} \geq \frac{4\sqrt{t}\ln t}{\Delta_i^2}, N_{i^\star,t}\geq \sqrt{t}\ln t) \Big]\nonumber\\
    &\leq \frac{4\sqrt{T}\ln T}{\Delta_i^2} + \sqrt{T}\ln T + \sum_{t=n}^T\PROB\Big[(I_{t+1} \in \set{S}_i | N_{i,t} \geq \frac{4\sqrt{t}\ln t}{\Delta_i^2}, N_{i^\star,t}\geq \sqrt{t}\ln t \Big]\nonumber\\
    &\leq \frac{4\sqrt{T}\ln T}{\Delta_i^2} + \sqrt{T}\ln T + \sum_{t=n}^T \Big( \rho_{i,t}(\frac{4\sqrt{t}\ln t}{\Delta_i^2}) + \rho_{i^\star, t}(\sqrt{t}\ln t)\Big)\; ,\label{thm:H-UCBgeneral:ineq6}
\end{align}
where inequality~\eqref{thm:H-UCBgeneral:ineq5} holds since if $N_{i,t} \geq 4\sqrt{t \ln t}/\Delta_i^2$ and $N_{i^\star, t} < \sqrt{t \ln t/4}$ then $R_{i,t} + \sqrt{\sqrt{t}\ln t / N_{i,t}} < 2 \leq \R_{i^\star, t} + \sqrt{\sqrt{t}\ln t / N_{i^\star, t}}$ and inequality~\eqref{thm:H-UCBgeneral:ineq6} holds by Lemma~\ref{thm:H-UCBgeneral:cl2}.

By integral test, we have:
\begin{align}
\sum_{t=2}^T \frac{1}{\sqrt{Ct}\ln t} &\leq \sum_{t=2}^T \sqrt{\frac{1}{Ct}} \leq \frac{2}{\sqrt{C}}\int_{t=2}^{T+1}\frac{d \sqrt{t}}{dt}dt \leq 2\sqrt{T/C}  \label{thm:H-UCBgeneral:ineq7}
\\
\sum_{t=2}^T 2L\ln T e^{-\frac{\sqrt{t}\ln t}{2L^2\ln^2T}}
&\leq \sum_{t=2}^T 2L\ln Te^{-\frac{\sqrt{t}}{2L^2\ln^2T}}
\leq \int_{t=0}^{\infty} 4L^4\ln^4T\frac{d(-2e^{-\sqrt{t}}(\sqrt{t}+1))}{dt} = 8L^5\ln^5T\; .
\label{thm:H-UCBgeneral:ineq8}
\end{align}

Applying inequalities~\eqref{thm:H-UCBgeneral:ineq7} and \eqref{thm:H-UCBgeneral:ineq8}, we can bound the summation of $\rho(C\sqrt{t} \ln t)$ as the following:
\begin{align}
\hspace{-1cm}\sum_{t=n}^T\rho_{i,t}(C\sqrt{t}\ln t)
&= \sum_{t=2}^T\Big[
 1/T + 2L\ln Te^{-\sqrt{t}\ln t /(2L^2\ln^2T)} + L\ln T\Big(\sum_{\substack{i_a\in \set{O}_i \\ \mu_{i_a} < \mu_{i_1}}} \frac{16\ln N_{i,t}}{\delta_{i_a}}  + (2+\frac{2\pi^2}{3})\delta_{i_a}\Big) / \sqrt{N_{i,t}\sqrt{t}\ln t}
\Big]\nonumber
\\
&\leq 1 + 8L^5\ln^5T + L \ln T\sum_{t=2}^T\Big[
\Big(\sum_{\substack{i_a\in \set{O}_i \\ \mu_{i_a} < \mu_{i_1}}}  \frac{16\ln N_{i,t}}{\delta_{i_a}}  + (2+\frac{2\pi^2}{3})\delta_{i_a}\Big) / \sqrt{N_{i,t}\sqrt{t} \ln t}
\Big]\nonumber
\\
&\leq 1 + 8L^5\ln^5T + L \ln T\sum_{t=2}^T\Big[
\Big(\sum_{\substack{i_a\in \set{O}_i \\ \mu_{i_a} < \mu_{i_1}}}  \frac{16\ln C + 16\ln t}{\delta_{i_a}}  + (2+\frac{2\pi^2}{3})\delta_{i_a}\big) / (\sqrt{Ct} \ln t)
\Big]\label{thm5:ineq1}
\\
&\leq 1 + 8L^5\ln^5T + \frac{2L\sqrt{T}\ln T}{\sqrt{C}}\sum_{\substack{i_a\in \set{O}_i \\ \mu_{i_a} < \mu_{i_1}}} (\frac{16\ln C + 16}{\delta_{i_a}} + 9\delta_{i_a})\; ,\nonumber
\end{align}
where we use $\ln t \leq 2\ln(\ln t)$ for $t \leq 2$ in inequality~\eqref{thm5:ineq1}.

Applying it to inequality~\eqref{thm:H-UCBgeneral:ineq6}, we can bound the expected number of selection as the following:
\begin{align*}
\EXP[N_{i,T}] &= \frac{4\sqrt{T}\ln T}{\Delta_i^2} + \sqrt{T}\ln T + \sum_{t=n}^T \rho_{i,t}(\frac{4\sqrt{t}\ln t}{\Delta_i^2}) + \rho_{i^\star, t}(\sqrt{t}\ln t) \\
&\leq \begin{aligned}[t]
  \frac{4\sqrt{T}\ln T}{\Delta_i^2} &+ \sqrt{T}\ln T \\
  &+ 1 + 8L^5 \ln^5 T + 2L\sqrt{T}\ln T\sum_{\substack{i^\star_a\in \set{O}_{i^\star} \\ \mu_{i^\star_a} < \mu_{i^\star_1}}} (\frac{16}{\delta_{i^\star_a}} + 9\delta_{i^\star_a})\\
  &+ 1 + 8L^5 \ln^5 T  + 2\Delta_iL\sqrt{T}\ln T\sum_{\substack{i_a\in \set{O}_i \\ \mu_{i_a} < \mu_{i_1}}} (\frac{16\ln (4/\Delta_i^2) + 16}{\delta_{i_a}} + 9\delta_{i_a})
  \end{aligned}.
\\
&\leq \begin{aligned}[t]
  \frac{4\sqrt{T}\ln T}{\Delta_i^2} &+ \sqrt{T}\ln T  + 2 + 16L^5\ln^5T\\
  &+ 2L\sqrt{T}\ln T\sum_{\substack{i^\star_a\in \set{O}_{i^\star} \\ \mu_{i^\star_a} < \mu_{i^\star_1}}} (\frac{16}{\delta_{i^\star_a}} + 9\delta_{i^\star_a}) + 2L\sqrt{T}\ln T\sum_{\substack{i_a\in \set{O}_i \\ \mu_{i_a} < \mu_{i_1}}} (\frac{48}{\delta_{i_a}} + 9\delta_{i_a})\; ,
  \end{aligned}
\end{align*}
where we use the fact $\ln(4x^2)/x \leq 2$ for any $x \geq 0$ and $\Delta_i \leq 1$ for the last inequality.

Hence, the expected regret from suboptimal agents can be upper bounded as the following:

\begin{align*}
\EXP [Regret_T^{sub} ]
&= \sum_{i \in \set{N}'} \EXP[N_{i,T}] \Delta_i^m \\
&\leq
\begin{aligned}[t]
  \frac{4\sqrt{T}\ln T}{\Delta_i^2}+ \sqrt{T}\ln T  &+ 2 + 16L^5\ln^5T\\
  &+ 2L\sqrt{T}\ln T\sum_{\substack{i^\star_a\in \set{O}_{i^\star} \\ \mu_{i^\star_a} < \mu_{i^\star_1}}} (\frac{16}{\delta_{i^\star_a}} + 9\delta_{i^\star_a}) + 2L\sqrt{T}\ln T\sum_{\substack{i_a\in \set{O}_i \\ \mu_{i_a} < \mu_{i_1}}} (\frac{48}{\delta_{i_a}} + 9\delta_{i_a})\Big]\; .
  \end{aligned}
\end{align*}

Combining it with inequality~\eqref{eq:reg_opt} and use the fact that any regret is at most $1$, we conclude the proof.
\end{proof}


{\bf Remark - Regret bound under logarithmic agent exploration.} 
Suppose that H-UCB is equipped with its phase 1 exploration term as $R_i + \sqrt{2\ln t /N_{i,t}}$.
Then, following the analogous step of our analysis, Lemma~\ref{thm:H-UCBgeneral:cl1} should be replaced into the bounding probability with $\sqrt{2\ln t / N_{i,t}}$, and this makes the bounding probability in equation~\eqref{thm:H-UCBgeneral:ineq3} be $O(\ln N_{i,t}) / O(\sqrt{N_{i,t} \ln t})$, where one can find that this finally leads to the undesirable result of linear regret upper bound.

{\bf Remark - Improving regret bound.}
We note that there exists some tuning points in our regret analysis. For each agent $i$, we can improve each arm's deviation probability and the internal regret bound by replacing phase 2 UCB index to KL-UCB index~\cite{mab:klucb1}. Secondly, we mainly upper-bound the probability that agent $i$ is selected given that agent $i$ and $i^\star$ is sampled sufficiently, take expectation over it, and directly derive the regret bound by multiplying it with $\Delta_i^m$. In this procedure, the regret bound might be tightened if we separately upper-bound the probability that agent $i$'s arm $i_a$ is selected given that agent $i$ and $i^\star$ is sampled sufficiently, since it will lead to the separation of final regret bound by the summation of expected number of count that $i_a$ is sampled multiplied by $\delta_{i_a}$.

\subsection{Proof of Theorem 6}
\begin{proof}
To derive a problem-independent regret upper-bound, we need to remove the dependency of our regret bound from $\delta_{i_a}$ and $\Delta_i$ in all the denominators. 
To handle $\delta_{i_a}$, we start with the inequality~\eqref{thm:H-UCBgeneral:ineq3} in Lemma~\ref{thm:H-UCBgeneral:cl1}.
Let $K = 1/c$.
For agent $i$, suppose that the arms are divided into the following two groups:
\begin{compactenum}[(1)]
\item Group 1 consists of almost optimal arm with $\delta_{i_a} < \sqrt{K\ln T\ln N_{i,t}/N_{i,t}}$\; .
\item Group 2 consists of arm with $\delta_{i_a} \geq \sqrt{K\ln T\ln N_{i,t}/N_{i,t}}$\; .
\end{compactenum}

Given $N_{i,t}$, the total expected internal regret of agent $i$ until $N_{i,t}$ is the sume of the regret of each group. Since group 1 only contains near-optimal arms, 
\begin{align*}
    \sum_{i_a \in \text{Group 1}}N_{i_a, t}\delta_{i_a} < \sqrt{K\ln T\ln N_{i,t}/N_{i,t}} \sum_{i_a \in \text{Group 1}} N_{i_a,t} \leq \sqrt{KN_{i,t}\ln N_{i,t}\ln T}\; .
\end{align*}

For the arms in group 2, the expected number of any arms in group 2 selected until round $N_{i,t}$ is,
\begin{align}
\sum_{i_a \in \text{Group 2}} \EXP[N_{i_a, t}\delta_{i_a, t}]
&\leq 
1 + K\ln T\Big(\sum_{\substack{i_a\in \set{O}_i \cap \text{Group 2}\\ \mu_{i_a} < \mu_{i_1}}}  \frac{8\ln N_{i,t}}{\delta_{i_a}}  + (1+\frac{\pi^2}{3})\delta_{i_a}\Big)\label{proof:thm7:ineq}
\\
&\leq 
8\sqrt{K N_{i,t}\ln N_{i,t}\ln T} + 5K\ln T + 1\; ,\label{proof:thm7:ineq'}
\end{align}
where inequality~\eqref{proof:thm7:ineq} follows from Lemma~\ref{lem:prob_noopt} and ~\ref{lem:conditional_replicator}.
Hence, we have 
\begin{align*}
\EXP[\sum_{a=1}^{l(i)}\delta_{i_a}N_{i_a, t}]&\leq 9\sqrt{1/cN_{i,t}\ln N_{i,t}\ln T} + 5L\ln N_{i,t} + 1\\
&\leq 9\ln T\sqrt{KN_{i,t}} + 5K\ln T + 1\; .
\end{align*}

Hence, inequality~\eqref{thm:H-UCBgeneral:ineq4} can be replaced into the following:
\begin{align*}
\PROB[\sum_{a=2}^{l(i)} |\mu_{i_1} - \mu_{i_a}|\cdot \frac{N_{i_a,t}}{N_{i,t}} > \sqrt{\frac{\sqrt{t}\ln t}{4N_{i,t}}}]
\leq 
\frac{9\ln T\sqrt{KNN_{i,t}} + 5K\ln T + 1}{\sqrt{N_{i,t}\sqrt{t} \ln t}}\; .
\end{align*}

On top of that, following the analogous arguments in the proof of Theorem 4 leads us to the following:
\begin{lemma}
At round $t$, if any suboptimal agent $i$ has been played at least $\frac{4\sqrt{t}\ln t}{\Delta_{i}^2}$ times and optimal agent $i^\star$ has been played at least $\sqrt{t}\ln t$, then $UCB_{i,t} < UCB_{i^\star,t}$ with probability at least $1- (\rho'_{i,t}(\frac{4\sqrt{t}\ln t}{\Delta_{i}^2}) + \rho'_{i^\star, t}(\sqrt{t}\ln t))$, i.e.
\begin{align}
    \PROB[I_{t+1} \in \set{S}_i | N_{i,t} \geq \frac{4\sqrt{t}\ln t}{\Delta_{i}^2}, N_{i^\star,t} \geq \sqrt{t}\ln t] \leq \rho'_{i,t}(\frac{4\sqrt{t}\ln t}{\Delta_i^2}) + \rho'_{i^\star, t}(\sqrt{t}\ln t)\; ,
\end{align}
where $\rho'_{i,t}(x) = 1/T + 2K\ln T\exp{(\frac{-\sqrt{t}\ln t}{2KN^2\ln^2T})}  + \frac{9\ln T\sqrt{Kx} + 5K\ln T + 1}{\sqrt{x\sqrt{t} \ln t}}$\; .
\end{lemma}
We skip the detailed proof since it can similarly be derived as the proof of Lemma~\ref{thm:H-UCBgeneral:cl2}.

By integral test, the following holds:
\begin{align}
\sum_{t=2}^T \frac{1}{(t\ln^2 t)^{1/4}} 
\leq \sum_{t=2}^T (\frac{1}{t})^{1/4}  \leq \int_{t=2}^{T+1} \frac{d (4/3t^{3/4}t)}{dt} dt \leq 4/3T^{3/4}\; .
\label{thm:H-UCBgeneral:ineq9}
\end{align}

By inequality~\eqref{thm:H-UCBgeneral:ineq7} and~\eqref{thm:H-UCBgeneral:ineq9}, the summation of $\rho'_{i,t}(C\sqrt{t \ln t})$ can be bounded as the following:
\begin{align*}
\sum_{t=2}^T \rho'_{i,t}(C\sqrt{t} \ln t)
&\leq 1 + 8K^5\ln^5T + \sum_{t=2}^T \Big(\frac{9\ln T\sqrt{KC\sqrt{t} \ln t} + 5K\ln T + 1}{\sqrt{Ct}\ln t}\Big)
\\
&\leq 1 + 8K^5\ln^5T + \sum_{t=2}^T \Big(\frac{9\ln T\sqrt{K}}{\sqrt{C\sqrt{t}\ln t}} +  \frac{5K\ln T + 1}{\sqrt{Ct}\ln t}\Big)
\\
&\leq 1 + 8K^5\ln^5T + 12\sqrt{K/C}T^{3/4}\ln T + 10/\sqrt{C}T^{1/2}\ln T+2\sqrt{T/C}\; .
\end{align*}

Following the analogous step to derive inequality~\eqref{thm:H-UCBgeneral:ineq6} in the proof of Theorem 5, we eventually have the following inequality for any suboptimal agent $i$:
\begin{align*}
\hspace{-1cm}\EXP[N_{i,T}]
&\leq
\frac{4\sqrt{T}\ln T}{\Delta_i^2} + \frac{\sqrt{T}\ln T}{2} + \sqrt{K}T^{3/4}\ln T(12 + 6\Delta_i) + 16L^5\ln^5T  + 5T^{1/2}\ln T(2 + 1\Delta_i) + \sqrt{T}(2 + \Delta_i)\\
&\leq  18\sqrt{K}T^{3/4}\ln T  + T^{1/2}\ln T(19 + \frac{4}{\Delta_i^2}) + 16L^5\ln^5T
\end{align*}

Now, let's divide the suboptimal agents into two groups:
For agent $i$, suppose that the arms are divided into the following two groups:
\begin{compactenum}[(1)]
\item Group 1 consists of almost optimal agents with $\Delta_i < (16n^2\ln^2 T/T)^{1/4}$\; .
\item Group 2 consists of arm with $\Delta_i \geq (16n^2\ln^2 T/T)^{1/4}$\; .
\end{compactenum}

Given $N_{i,t}$, the total expected regret is the sum of the regret in these two groups.
Since group 1 only contains near-optimal agents, 
\begin{align*}
    \sum_{i \in \text{Group 1}}N_{i,T}\Delta_i \leq (16n^2\ln^2 T/T)^{1/4}\sum_{i \in \text{Group 1}} N_{i,t} \leq  2\sqrt{n}T^{3/4}\ln^{1/2}T\; .
\end{align*}

For the agents in group 2,
\begin{align*}
\sum_{i \in \text{Group 2}}\EXP[N_{i,T}]\Delta_i
&\leq
\sum_{i \in \text{Group 2}} \Big[ 18\sqrt{K}T^{3/4}\ln T  + T^{1/2}\ln T(19 + \frac{4}{\Delta_i}) + 16K^5\ln^5T \Big]
\\
&\leq T^{3/4}\ln T(18n\sqrt{K} + 19n + 2\sqrt{n}) + 16nK^5\ln^5T\; .
\end{align*}

Summing the regret bound for group 1 and 2 yields,
\begin{align*}
\sum_{i \in N'}\EXP[N_{i,T}]\Delta_i = T^{3/4}\ln T(18n\sqrt{K} + 19n + 4\sqrt{n}) + 16nL^5\ln^5T.
\end{align*}
From inequality~\eqref{proof:thm7:ineq'}, we already have the problem-independent regret upper-bound for agent in $N^\star$ as follow,
\begin{align*}
\sum_{i \in N^\star} \EXP[\sum_{a=1}^{l(i)}\delta_{i_a}N_{i_a, t}]&\leq \sum_{i \in N^\star} \Big(9\ln T\sqrt{KT} + 5K\ln T + 1 \Big)\\
&\leq 9n\sqrt{K}T^{1/2}\ln T + 5nK\ln T + n\; .
\end{align*}
Hence summing up the results give us the following regret bound:
\begin{align*}
 T^{3/4}\ln T(18n\sqrt{K} + 19n + 2\sqrt{n}) + 9n\sqrt{K}T^{1/2}\ln T +16nK^5\ln^5T +  5nK\ln T + n = O(\frac{n}{\sqrt{c}}T^{\frac{3}{4}}\ln T)\; ,
\end{align*}
and we conclude the proof.
\end{proof}

{\bf Remark - honest agent.}
Suppose that {\em honest agent}, who honestly register all the original arms without any replication, exists rather than replicators. In this case, we can get rid of the arm sampling part in RH-UCB, and correspondingly, the order of nominator in inequality~\eqref{thm:H-UCBgeneral:ineq4} can be bounded by $O(\ln N_{i,t})$.
In this case, following the analogous step provided in the analysis, we can find that eventual problem-dependent regret bound under this scenario will be $O(\sqrt{T \ln T})$, and problem-independent bound will be $O(T^{3/4}\ln^{1/4}T)$.
Note that all the analysis can also be generalized into the scenario when the strategic agent, honest agent, and replicator coexist, where the order of regret will be the same as that of the case when strategic agent and replicator exist since honest agent can be regarded as a weaker version of replicator.

\section{Prior-free Robust Hierarchical UCB and its regret analysis}
\label{sec:prh-ucb}
We provide a detailed description of Prior-free Robust Hierarchical UCB (PRH-UCB) which we briefly described in Section 4, and present corresponding regret analysis.
We note that PRH-UCB still achieves replication-proofness, where we skip the proof since it can analogously be derived from the proof of Theorem 2 since it inherits the hierarchical structure of H-UCB.
Note that we tune {\tt Phase 1} exploration parameter to $\sqrt{t\ln^3 t}$ to optimize the order of poly-logarithmic term in regret bound.
For the regret analysis, we only provide a proof sketch since the detailed steps are analogous to the proof of Theorem 5 and 6.

\begin{theorem}
PRH-UCB has $O(\frac{n}{c^2}T^{\frac{1}{2}}\ln^{\frac{3}{2}}T)$ problem-dependent and $O(\frac{n}{\sqrt{c}}T^{\frac{3}{4}}\ln^{\frac{3}{2}}T)$ problem-independent regret upper-bound for any $T$ such that $T\ln^3 T \geq e^{2/c}$.
\end{theorem}

\begin{algorithm} 
\SetAlgoLined

Initialize $R(i), N(i), r(a)$ and $n(a)$ at zero and $\set{B}_i$ at an empty set
for all $i \in \set{N}$ and $a\in \cup_{i \in \set{N}} \set{S}_i$;

 \For{$t=1,2,\ldots$}{
  \tcp{Phase 1 - agent selection}
  \eIf{Unexplored agent exists}{ 
    Pick an unexplored $\hat{i}\in \{i \in \set{N}: N(i) = 0\} $; 
  }{
    Pick $\hat{i} = \argmax_{i \in \set{N}} \left\{R(i) + \sqrt{\frac{\sqrt{t\ln^3t}}{N(i)}} \right\}$;
  }
  
  \tcp{Phase 2 - arm selection}
  \eIf{$|\set{B}_{\hat{i}}| < \min(|\set{S}_{\hat{i}}|, \ln^2 t)$}{
    Pick an unexplored arm $\hat{a} \in \{a \in \set{S}_{\hat{i}} \setminus \set{B}_{\hat{i}}: n(a) = 0\}$;\\
    $\set{B}_{\hat{i}} \leftarrow \set{B}_{\hat{i}} \cup \{\hat{a}\}$;
  }{
    Pick $\hat{a} = \argmax_{a \in \set{B}_{\hat{i}}} 
    \left\{r(a) + \sqrt{\frac{2 \ln N(i)}{n(a)}} \right\}$;
  }
  
  Play $A_t = \hat{a}$ and receive reward $R_t$; \\
  Update statistics:
  $r(\hat{a}) \leftarrow \tfrac{r(\hat{a})n(\hat{a}) + R_t}{n(\hat{a})+1}; ~~
  R(\hat{i}) \leftarrow \tfrac{R(\hat{i})N(\hat{i}) + R_t}{N(\hat{i})+1};$
  $n(\hat{a}) \leftarrow n(\hat{a}) + 1; ~~ N(\hat{i}) \leftarrow N(\hat{i}) + 1;$
 }
 \caption{Prior-free Robust Hierarchical UCB (PRH-UCB)\label{alg:prhucb}}
\end{algorithm}

  

\begin{proof}[Proof sketch]
The proof structure is exactly the same with the proof of Theorem 5 and Theorem 6. We briefly explain how the order of magnitude will be changed when the sampling amount becomes $\ln^2T$.
Given round $t$, we redefine the random variable $Y_i$ to be the event that $\set{B}_i$ contains $i$'s optimal arm or any replica of it until round $t$.
Then we can derive the following lemma:
\begin{lemma}\label{lem:prob_noopt2}
At round $t$, if $N_{i,t} \geq e^{1/c}$, then $\PROB[Y_i^c] \leq 1/N_{i,t}$.
\end{lemma}
\begin{proof}[Proof of Lemma~\ref{lem:prob_noopt2}]
At round $t$, agent $i$ is selected $N_{i,t}$ times and $\min(|\set{S}_i|, \ln^2 N_{i,t})$ arms will be sampled until then.
If $|\set{S}_i| \leq \ln^2 N_{i,t}$, then the optimal arm will be contained in $\set{B}_i$.
Otherwise, the following inequalities conclude the proof.
\begin{align*}
    \PROB[Y_i^c] 
    &\leq (1-c)\frac{|\set{S}_i|(1-c) -1}{|\set{S}_i|-1}\frac{|\set{S}_i|(1-c) -2}{|\set{S}_i|-2}\ldots \frac{|\set{S}_i|(1-c) -\ln^2 N_{i,t}}{|\set{S}_i|-\ln^2 N_{i,t}}
    \\
    &\leq (1-c)^{\ln^2 N_{i,t}} \\
    &\leq (1-c)^{1/c\ln N_{i,t}} \\
    &\leq 1/N_{i,t}\; ,
\end{align*}
where the last inequality holds from $(1-1/x)^x \leq 1/e$ for any $x \geq 0$.
\end{proof}

Then, Lemma~\ref{lem:conditional_replicator} should be replaced as the following:
\begin{lemma}\label{lem:conditional_replicator2}
At round $t$, conditional expected internal regret given that $Y_i$ is occurred can be bounded as the following:
\begin{align*}
\sum_{i_a \in \set{B}_i}\EXP [\delta_{i_a} N_{i_a,t} |  Y_i] \leq \ln^2 N_{i ,t}\Big(\sum_{\substack{i_a\in \set{O}_i \\ \mu_{i_a} < \mu_{i_1}}}  \frac{8\ln N_{i,t}}{\delta_{i_a}}  + (1+\frac{\pi^2}{3})\delta_{i_a}\Big)\; .
\end{align*}
\end{lemma}
\begin{proof}[Proof of Lemma~\ref{lem:conditional_replicator2}]
Given $t$ and $N_{i,t}$, $|\set{B}|_i = \min(|\set{S}_i|, \ln^2 N_{i,t})$.
If $|\set{S}_i| \leq \ln^2 N_{i,t}$, then the optimal arm will be contained in $\set{B}_i$, and the result directly holds.
Otherwise, we have
\begin{align*}
\EXP[\delta_{i_a} N_{i_a,t} | Y_i] \leq (1-\PROB[Y_i])\ln^2 N_{i, t}\Big(\sum_{\substack{a\in \set{O}_i \\ \mu_{i_a} < \mu_{i_1}}}\big( \frac{8 \ln N_{i,t}}{\delta_{i_a}} + (1+\frac{2\pi^2}{3})\delta_{i_a}\big) \Big)\; ,
\end{align*}
and it concludes the proof.
\end{proof}


Now we present two lemmas which are analogous to Lemma~\ref{thm:H-UCBgeneral:cl1} and~\ref{thm:H-UCBgeneral:cl2}.
\begin{lemma}
Given any agent $i \in \set{N}$, for each arm $j$ in $j \in \set{S}_i$ at time $t$, we have
\begin{align*}
    \PROB\Big[| R_{i,t} - \mu_{i_1}| \geq \sqrt{\frac{\sqrt{t\ln^3 t}}{N_{i,t}}}\Big] \leq \rho''(N_{i,t})\; ,
\end{align*}
where 
\begin{align*}
    \rho''(N_{i,t}) = 1/N_{i,t} + \frac{1}{c}\ln Te^{-\sqrt{t\ln^3 t} /(2L^2\ln^2T)} + \frac{1}{c}\ln^2 N_{i,t}\Big(\sum_{\substack{i_a\in \set{O}_i \\ \mu_{i_a} < \mu_{i_1}}}  \frac{16\ln N_{i,t}}{\delta_{i_a}}  + (2+\frac{2\pi^2}{3})\delta_{i_a}\Big) / \sqrt{N_{i,t}\sqrt{t\ln^3 t}}\; .
\end{align*}
\end{lemma}

\begin{lemma}
At round $t$ such that $t \ln^3 t \geq e^{2/c}$, if any suboptimal replicator $i$ has been played at least $\frac{4\sqrt{t\ln^3 t}}{\Delta_{i}^2}$ times and optimal replicator $i^\star$ has been played at least $\sqrt{t\ln^3 t}$, then $UCB_{i,t} < UCB_{i^\star,t}$ with probability at least $1- (\rho''_{i,t}(\frac{4\sqrt{t\ln^3 t}}{\Delta_{i}^2}) + \rho''_{i^\star, t}(\sqrt{t\ln^3 }t))$, i.e.
\begin{align}
    \PROB\Big[I_{t+1} \in \set{S}_i | N_{i,t} \geq \frac{4\sqrt{t\ln^3 t}}{\Delta_{i}^2}, N_{i^\star,t} \geq \sqrt{t\ln^3 t}\Big] \leq \rho''_{i,t}(\frac{4\sqrt{t\ln^3 t}}{\Delta_i^2}) + \rho''_{i^\star, t}(\sqrt{t\ln^3 t})\; .
\end{align}
\end{lemma}

Then, followed by some calculations, one can find that $\rho''(\sqrt{t\ln^3 t}) = O(\sqrt{t\ln^3 t})$ for any $t$ such that $t\ln^3 t \geq e^{2/c}$.
These enables us to obtain $O(\frac{n}{c^2}\sqrt{T\ln^3 T})$ problem-dependent regret upper bound and $O(\frac{n}{\sqrt{c}}T^{3/4}\ln^{3/2}T)$ problem-independent regret upper bound. Note that deriving problem-independent regret bound requires a similar technique as the ones used in the proof of Theorem 6.
 

\end{proof}

\clearpage


\end{document}